\documentclass[final,12pt]{colt2022} 

\newcommand{\R}{\mathbb{R}}

\newcommand{\E}{\mathbb{E}}

\renewcommand{\vec}[1]{\mathbf{#1}}
\DeclareMathOperator{\var}{Var}

\newcommand{\assmcore}[1]{{#1}^*}
\newcommand{\stimcore}[1]{S_{#1}}
\newcommand{\support}[1]{\bar{#1}}

\title[Learning with Assemblies of Neurons]{Assemblies of Neurons Learn to Classify Well-Separated Distributions}
\usepackage{times}

\coltauthor{%
 \Name{Max Dabagia} \Email{maxdabagia@gatech.edu}\\
 \addr School of Computer Science, Georgia Tech
 \AND
 \Name{Christos H. Papadimitriou} \Email{christos@columbia.edu}\\
 \addr Department of Computer Science, Columbia University
 \AND
 \Name{Santosh S. Vempala} \Email{vempala@gatech.edu}\\
 \addr School of Computer Science, Georgia Tech
}

\begin{document}

\maketitle

\begin{abstract}%
    \noindent An assembly is a large population of neurons whose synchronous firing is hypothesized to represent a memory, concept, word, and other cognitive categories.  Assemblies are believed to provide a bridge between high-level cognitive phenomena and low-level neural activity. Recently, a computational system called the \emph{Assembly Calculus} (AC), with a repertoire of biologically plausible operations on assemblies, has been shown capable of simulating arbitrary space-bounded computation, but also of simulating complex cognitive phenomena such as language, reasoning, and planning.  However, the mechanism whereby assemblies can mediate {\em learning} has not been known.  Here we present such a mechanism, and prove rigorously that, for simple classification problems defined on distributions of labeled assemblies, a new assembly representing each class can be reliably formed in response to a few stimuli from the class; this assembly is henceforth reliably recalled in response to new stimuli from the same class.  Furthermore, such class assemblies will be distinguishable as long as the respective classes are reasonably separated --- for example, when they are clusters of similar assemblies, or more generally separable with margin by a linear threshold function. To prove these results, we draw on random graph theory with dynamic edge weights to estimate sequences of activated vertices, yielding strong generalizations of previous calculations and theorems in this field over the past five years. These theorems are backed up by experiments demonstrating the successful formation of assemblies which represent concept classes on synthetic data drawn from such distributions, and also on MNIST, which lends itself to classification through one assembly per digit. Seen as a learning algorithm, this mechanism is entirely online, generalizes from very few samples, and requires only mild supervision --- all key attributes of learning in a model of the brain.  We argue that this learning mechanism, supported by separate sensory pre-processing mechanisms for extracting attributes, such as edges or phonemes, from real world data, can be the basis of biological learning in cortex. 
\end{abstract}


\section{Introduction}

The brain has been a productive source of inspiration for AI, from the perceptron and the neocognitron to deep neural nets. Machine learning has since advanced to dizzying heights of analytical understanding and practical success, but the study of the brain has lagged behind in one important dimension: After half a century of intensive effort by neuroscientists (both computational and experimental), and despite great advances in our understanding of the brain at the level of neurons, synapses, and neural circuits, we still have no plausible mechanism for explaining intelligence, that is, the brain's performance in planning, decision-making, language, etc. As Nobel laureate Richard Axel put it, ``we have no logic for translating neural activity into thought and action'' \citep{AxelNeuron2018}.

Recently, a high-level computational framework was developed with the explicit goal to fill this gap: the Assembly Calculus (AC)  \citep{papadimitriou2020brain}, a computational model whose basic data type is the {\em assembly of neurons}.  Assemblies, called ``the alphabet of the brain'' \citep{buzsaki2019brain}, are large sets of neurons whose simultaneous excitation is tantamount to the subject's thinking of an object, idea, episode, or word (see \citet{Josh2016}). Dating back to the birth of neuroscience, the ``million-fold democracy'' by which groups of neurons act collectively without central control was first proposed by \citet{sherrington1906integrative} and was the empirical phenomenon that Hebb attempted to explain with his theory of plasticity \citep{hebb1949}. Assemblies are initially created to record memories of external stimuli \citep{quiroga2016neuronal}, and are believed to be subsequently recalled, copied, altered, and manipulated in the non-sensory brain \citep{Josh2012, Buzsaki:10}.  The Assembly Calculus provides a repertoire of operations for such manipulation, namely {\tt project, reciprocal-project, associate, pattern-complete,} and {\tt merge} encompassing a complete computational system.  Since the Assembly Calculus is, to our knowledge, the only extant computational system whose purpose is to bridge the gap identified by Axel in the above quote \citep{AxelNeuron2018}, it is of great interest to establish that complex cognitive functions can be plausibly expressed in it.  Indeed, significance progress has been made over the past year, see for example \citet{mitropolsky2021biologically} for a parser of English and  \citet{damore2021planning} for a program mediating planning in the blocks world, both written in the AC programming system.  Yet despite these recent advances, one fundamental question is left unanswered: If the Assembly Calculus is a meaningful abstraction of cognition and intelligence, why does it not have a {\tt learn} command?  {\em How can the brain learn through  assembly representations?}  

This is the question addressed and answered in this paper. As assembly operations are a new learning framework and device, one has to start from the most basic questions: Can this model classify assembly-encoded stimuli that are separated through clustering, or by half spaces? Recall that learning linear thresholds is a theoretical cornerstone of supervised learning, leading to a legion of fundamental algorithms: Perceptron, Winnow, multiplicative weights, isotron, kernels and SVMs, and many variants of gradient descent. 

Following \citet{papadimitriou2020brain}, we model the brain as a directed graph of excitatory neurons with dynamic edge weights (due to plasticity).  The brain is subdivided into {\em areas}, for simplicity each containing $n$ neurons connected through a $G_{n,p}$ random directed graph \citep{erdHos1960evolution}. Certain ordered pairs of areas are also connected, through random bipartite graphs.  We assume that neurons fire in discrete time steps. At each time step, each neuron in a brain area will fire 
if its synaptic input from the firings of the previous step is among the top $k$ highest out of the $n$ neurons in its brain area.  This selection process is called {\em $k$-cap}, and is an abstraction of the process of {\em inhibition} in the brain, in which a separate population of inhibitory neurons is induced to fire by the firing of excitatory neurons in the area, and through negatively-weighted connections prevents all but the most stimulated excitatory neurons from firing. Synaptic weights are altered via Hebbian plasticity and homeostasis (see Section 2 for a full description).  In this stylized mathematical model, reasonably consistent with what is known about the brain, it has been shown that the operations of the Assembly Calculus converge and work as specified (with high probability relative to the underlying random graphs). These results have also been replicated by simulations in the model above, and also in more biologically realistic networks of spiking neurons (see \citet{LMPV,papadimitriou2020brain}). {\em In this paper  we develop, in the same model, mechanisms for learning to classify well-separated classes of stimuli}, including clustered distributions and linear threshold functions with margin. Moreover, considering that the ability to learn from few examples, and with mild supervision, are crucial characteristics of any brain-like learning algorithm, we show that learning with assemblies does both quite naturally.

\section{A mathematical model of the brain} 
Here we outline the basics of the model in \citet{papadimitriou2020brain}.  There are a finite number $a$ of brain areas, denoted $X, Y, \ldots$ (but in this paper, we will only need one brain area where learning happens, plus another area where the stimuli are presented).  Each area is a random directed graph with $n$ nodes called {\em neurons} with each directed edge present independently with probability $p$ (for simplicity, we take $n$ and $p$ to be the same across areas).  Some ordered pairs of brain areas are also connected by random bipartite graphs, with the same connection probability $p$.  Importantly, each area may be {\em inhibited}, which means that its neurons cannot fire; the status of the areas is determined by explicit {\tt inhibit/disinhibit} commands of the AC. \footnote{The brain's neuromodulatory systems \citep{jones2003arousal, harris2011cortical} are plausible candidates to implement these mechanisms.} This defines a large random graph $G=(N,E)$ with $|N|=an$ nodes and a number $|E|$ of directed edges which is in expectation $(a+b)pn^2-apn$, where $b$ is the number of pairs of areas that are connected. Each edge $(i,j)\in E$ in this graph, called a {\em synapse}, has a dynamic non-negative weight $w_{ij}(t)$, initially $1$. 

This framework gives rise to a discrete-time dynamical system, as follows: The {\em state} of the system at any time step $t$ consists of (a) a bit for each area $X$, inh$(X,t)$, initially $0$, denoting whether the area is inhibited; (b) a bit for each neuron $i$, fires$(i,t)$, denoting whether $i$ spikes at time $t$ (zero if the area $X$ of $i$ has inh$(X,t)=1$); and (c) the weights of all synapses $w_{ij}(t)$, initially one.  

The state transition of the dynamical system is as follows: For each neuron $i$ in area $A$ with inh$(X,t+1)=0$ (see the next paragraph for how inh$(X,t+1)$ is determined), define its {\em synaptic input} at time $t+1$, 
\[
\mbox{SI}(i,t+1) = \sum_{(j,i)\in E} \hbox{\rm fires}(j,t)w_{ji}(t).
\]
For each $i$ in area $X$ with inh$(X,t+1)=0$, we set fires$(i,t)=1$ iff $i$ is among the $k$ neurons in its area that have highest SI$(i,t+1)$ (breaking ties arbitrarily).  This is the {\em $k$-cap} operation, a basic ingredient of the AC framework, modeling the inhibitory/excitatory balance of a brain area.\footnote{\citet{binas2014learning} showed rigorously how a $k$-cap dynamic could be emerge in a network of excitatory and inhibitory neurons.}  As for the synaptic weights, 
\[
w_{ji}(t+1)=w_{ji}(t)(1+\beta\cdot\hbox{\rm fires}(j,t)\hbox{\rm fires}(i,t+1)).
\]
That is, if $j$ fires at time $t$ and $i$ fires at time $t+1$, Hebbian plasticity dictates that $w_{ji}$ be increased by a factor of $1+\beta$ at time $t+1$. So that the weights do not grow unlimited, a {\em homeostasis} process renormalizes, at a slower time scale, the sum of weights along the incoming synapses of each neuron (see \citet{davis2006homeostatic} and \citet{turrigiano2011too} for reviews of this mechanism in the brain).

Finally, the AC is a computational system {\em driving} the dynamical system by executing commands at each time step $t$ (like a programming language driving the physical system that is the computer's hardware).  The AC commands {\tt (dis)inhibit$(X)$} change the inhibition status of an area at time $t$; and the command {\tt fire$(x)$,} where $x$ is the name of an assembly (defined next) in a disinhibited area, overrides the selection by $k$-cap, and causes the $k$ neurons of assembly $x$ to fire at time $t$.  

An {\em assembly} is a highly interconnected (in terms of both number of synapses and their weights) set of $k$ neurons {\em in an area} encoding a real world entity. Initially, assembly-like representations exist only in a special {\em sensory area,} as representations of perceived real-world entities such as a heard (or read) word. Assemblies in the remaining, non-sensory areas are {\em an emergent behavior} of the system, copied and re-copied, merged, associated, etc.,~through further commands of the AC.  This is how the model is able to simulate arbitrary $\frac{n}{k}$ space bounded computations \citep{papadimitriou2020brain}. The most basic such command is {\tt project$(x,Y,y)$,} which, starting from an assembly $x$ in area $X$, creates in area $Y$ (where there is connectivity from $X$ to $Y$) a new assembly $y$, which has strong synaptic connectivity from $x$ and which will henceforth fire every time $x$ fires in the previous step, and $Y$ is not inhibited.  This command entails disinhibiting the areas $X,Y$, and then {\tt firing} (the neurons in) assembly $x$ for the next $T$ time steps.  It was shown by \citet{papadimitriou2019random} that, with high probability, after a small number of steps, a stable assembly $y$ in $Y$ will emerge, which is densely intraconnected and has high connectivity from $x$. The mechanism achieving this convergence involves synaptic input from $x$, which creates an initial set $y^1$ of firing neurons in $Y$, which then evolves to $y^t, t=2,\ldots$ through sustained synaptic input from $x$ {\em and recurrent input from $y^{t-1}$}, while these two effects are further enhanced by plasticity.  

Incidentally, this convergence proof (see \cite{legenstein2018long, papadimitriou2019random, papadimitriou2020brain} for a sequence of sharpened versions of this proof over the past years) is the most mathematically sophisticated contribution of this theory to date.  The theorems of the present paper can be seen as substantial {\em generalizations} of that result:  Whereas in previous work an assembly is formed as a copy of one stimulus firing repeatedly (memorization), so that this new assembly will henceforth fire whenever the same stimulus us presented again, in this paper we show rigorously that an assembly will be formed in response to the sequential firing of many stimuli, all drown from the same distribution (generalization), and the formed assembly will fire reliably every time another stimulus from the same distribution is presented. 
\begin{figure}[ht]
    \centering
    \includegraphics[width=0.9\linewidth]{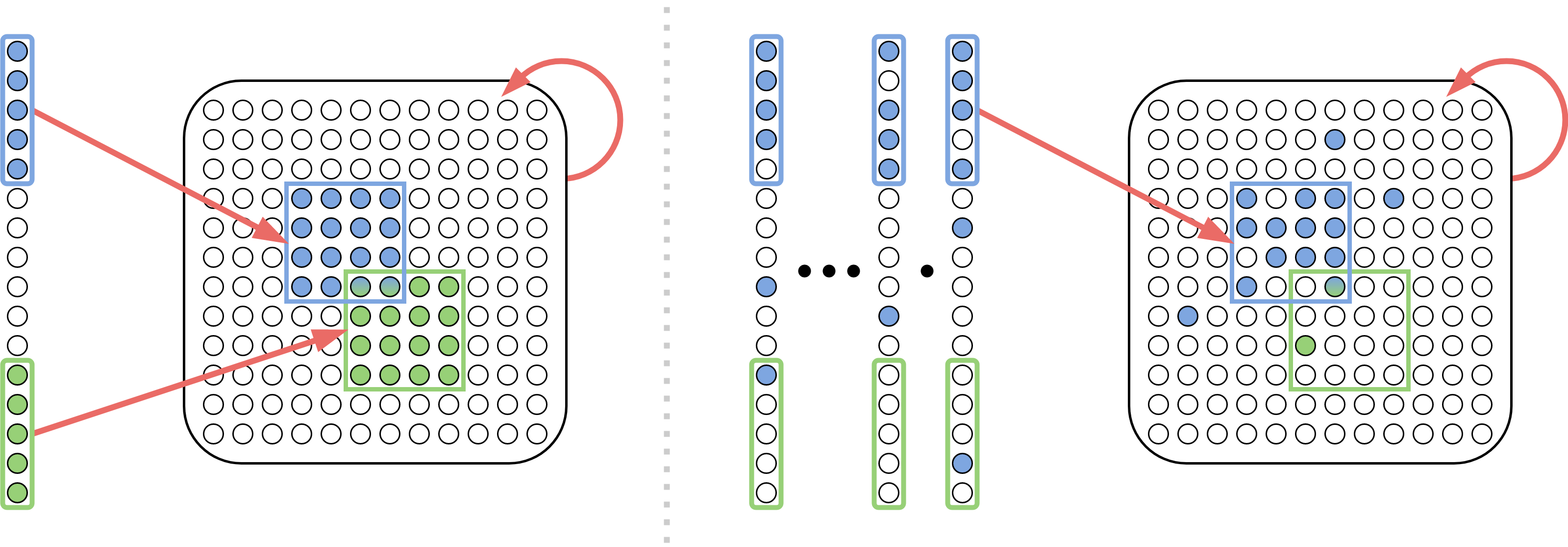}
    \caption{A mathematical model of learning in the brain. Our model (left) has a sensory area (column) connected to a brain area (square), both made up of spiking neurons. Two different stimuli classes (with their core sets in blue/green on the left) project from the sensory area via synaptic connections (arrows). Assemblies in the brain area (with core sets in the corresponding colors) form in response to these stimuli classes, each of which consistently fires when a constant fraction of the associated stimuli class's core set does. Our learning algorithm (right) consists of presenting a stream of stimuli from each class.
    }
    \label{fig:schematic}
\end{figure}

The key parameters of our model are $n, k, p$, and $\beta$. Intended values for the brain are $n = 10^7, k= 10^4, p = 10^{-3}, \beta=0.1$, but in our simulations we have also had success on a much smaller scale, with $n=10^3, k=10^2, p=10^{-1}, \beta =0.1$. $\beta$ is an important parameter, in that adequately large values of $\beta$ guarantee the convergence of the AC operations.
For a publicly available simulator of the Assembly Calculus (in which the Learning System below can be readily implemented) see \url{http://brain.cc.gatech.edu}.

\paragraph{The learning mechanism.} 
For the purpose of demonstrating learning within the framework of AC, we consider the specific setting described below. First, there is a special area, called the {\em sensory area,} in which training and testing data are encoded as assembly-like representations called {\em stimuli}. There is only one other brain area (besides the sensory area), and that is where learning happens, through the formation of assemblies in response to sequences of stimuli.

A {\em stimulus} is a set of about $k$ neurons firing simultaneously (``presented'') in the sensory area. Note that, exceptionally in the sensory area, a number of neurons that is a little different from $k$ may fire at a step.  A {\em stimulus class} $A$ is a distribution over stimuli, defined by three parameters: two scalars $r,q \in [0,1], r > q$, and a set of $k$ neurons $\stimcore{A}$ in the sensory area.  To generate a stimulus $x \in \{0,1\}^n$ in the class $A$, each neuron $i \in \stimcore{A}$ is chosen with probability $r$, while for each $i \not\in \stimcore{A}$, the probability of choosing neuron $i$ is $qk/n$. It follows immediately that, in expectation, an $r$ fraction of the neurons in the stimulus core are set to $1$ and the number of neurons outside the core that are set to $1$ is also $O(k)$. 

The presentation of a sequence of stimuli from a class $A$ in the sensory area evokes in the learning system a {\em response} $R$, a {\em distribution over assemblies} in the brain area.  
We show that, as a consequence of plasticity and $k$-cap, this distribution $R$ will be highly concentrated, in the following sense: Consider the set $S_R$ of all assemblies $x$ that have positive probability in $R$.  Then the numbers of neurons in both the intersection $\assmcore{R}=\bigcap_{x\in S_R} x$, called the {\em core} of $R$ and the union $\support{R}=\bigcup_{x\in S_R} x$ are close to $k$, in particular $k-o(k)$ and $k+o(k)$ respectively.\footnote{The larger the plasticity, the closer these two values are (see \citet{papadimitriou2019random}, Fig. 2).}  
In other words, neurons in $\assmcore{R}$ fire far more often on average than neurons in $\support{R} \setminus \assmcore{R}$.

Finally, our learning protocol is this: Beginning with the brain area at rest, stimuli are repeatedly sampled from the class, and made to fire. 
After a small number of training samples, the brain area returns to rest, and then the same procedure is repeated for the next stimulus class, and so on.  Then testing stimuli are presented in random order to test the extent of learning.
(see Algorithm 1 in an AC-derived programming language.)
\if{false}
At a particular moment in time, the state of the network can be described as $\mathcal B = (x, y, A, W)$, where $x \in \{0, 1\}^n$ is the activations of neurons in the sensory area, $y \in \{0, 1\}^n$ is the activations of neurons in the learning area (with $\sum_{i} y_i \le k$), $A$ is a set of weighted directed edges from  neurons in $X$ to neurons in $Y$, and $W$ is a set of weighted directed edges from neurons in $Y$ to other neurons in $Y$, excluding loops. To interact with the brain area, we provide a few basic commands: 

\begin{itemize}
    \item $\texttt{input}(\mathcal B, x)$ updates the current value of the sensory area activations (the input) to $x$
    \item $\texttt{step}(\mathcal B)$ progresses to the next time step: $\mathcal B$ computes the new value $y'$ of $y$ based on the current value of $(x, y, A, W)$, and updates $A$ and $W$ using Hebbian plasticity.
    \item $\texttt{read}(\mathcal B)$ returns $y$, i.e. the indicator vector of the neurons in the learning area which are currently firing
    \item $\texttt{inhibit}(\mathcal B)$ forces $y$ to zero, silencing all activity in the learning area
\end{itemize}
\fi

\begin{algorithm} \label{alg:mechanism}
\caption{The learning mechanism. ($B$ denotes the brain area.)}
\KwIn{a set of stimulus classes $A_1, \ldots, A_c$; $T \ge 1$}
\KwOut{A set of assemblies $y_1, \ldots, y_c$ in the brain area encoding these classes}
\ForEach{ stimulus class $i$}{
    inh$(B) \gets 0$\;\\
    \ForEach{ time step $1 \le t \le T$}{
        Sample $x \sim A_i$ and fire $x$\; 
        }
    $y_i \gets \texttt{read}(B)$\;\\
    inh$(B) \gets 1$\;
    }
\end{algorithm}

That is, we sample $T$ stimuli $x$ from each class, fire each $x$ to cause synaptic input in the brain area, and after the $T$th sample has fired we record the assembly which has been formed in the brain area.  This is the representation for this class.

\subsection*{Related work}
There are numerous learning models in the neuroscience literature. 
In a variation of the model we consider here, \citet{RangamaniGandhi2020} have considered supervised learning of Boolean functions using assemblies of neurons, by setting up separate brain areas for each label value.
Amongst other systems with rigorous guarantees, assemblies are superficially similar to the ``items'' of Valiant's neuroidal model \citep{Valiant94}, in which supervised learning experiments have been conducted \citep{Valiant00, FeldmanV09}, where an output neuron is clamped to the correct label value, while the network weights are updated under the model. The neuroidal model is considerably more powerful than ours, allowing for arbitrary state changes of neurons and synapses; in contrast, our assemblies rely on only two biologically sound mechanisms, plasticity and inhibition. 

Hopfield nets \citep{hopfield1982neural} are recurrent networks of neurons with symmetric connection weights which will converge to a memorized state from a sufficiently similar one, when properly trained using a local and incremental update rule.  In contrast, the memorized states our model produces (which we call assemblies) emerge through plasticity and randomization from the structure of a random directed network, whose weights are asymmetric and nonnegative, and in which inhibition --- not the sign of total input --- selects which neurons will fire. 

Stronger learning mechanisms have recently been proposed. Inspired by the success of deep learning, a large body of work has shown that cleverly laid-out microcircuits of neurons can approximate backpropagation to perform gradient descent \citep{Lillicrap2016RandomSF, sacramento2017dendritic, guerguiev2017towards, sacramento2018dendritic, whittington2019theories, lillicrap2020}. These models rely crucially on novel types of neural circuits which, although biologically possible, are not presently known or hypothesized in neurobiology, nor are they proposed as a theory of the way the brain works. These models are capable of matching the performance of deep networks on many tasks, which are more complex than the simple, classical learning problems we consider here. The difference between this work and ours is, again, that here we are showing that learning arises naturally from well-understood mechanisms in the brain, in the context of the assembly calculus.

\section{Results} \label{results}
Very few stimuli sampled from an input distribution are activated sequentially at the sensory area.   
The only form of supervision required is that all training samples from a given class are presented consecutively. 
Plasticity and inhibition alone ensure that, in response to this activation, an assembly will be formed for each class, and that this same assembly will be recalled at testing upon presentation of other samples from the same distribution. In other words, learning happens. 
And in fact, despite all these limitations, we show that the device is an efficient learner of interesting concept classes.

Our first theorem is about the creation of an assembly in response to inputs from a stimulus class. This is a generalization of a theorem from \citet{papadimitriou2019random}, where the input stimulus was held constant; here the input is a stream of random samples from the same stimulus class. Like all our results, it is a statement holding with high probability (WHP), where the underlying random event is the random graph and the random samples. When sampled stimuli fire, the assembly in the brain area changes. The neurons participating in the current assembly (those whose synaptic input from the previous step is among the $k$ highest) are called the current {\em winners.}
A {\em first-time winner} is a current winner that participated in no previous assembly (for the current stimulus class).

\begin{theorem}[Creation] \label{theorem:creation} 
Consider a stimulus class $A$ projected to a brain area. Assume that 
\[\beta \geq \beta_0 = \frac{1}{r^2}\frac{\left(\sqrt{2} - r^2\right)\sqrt{2\ln\left(\frac{n}{k}\right)} + \sqrt{6}}{\sqrt{kp} + \sqrt{2\ln\left(\frac{n}{k}\right)}}\] Then WHP no first-time winners will enter the cap after $O(\log k)$ rounds, and moreover the total number of winners $\support{A}$  can be bounded as \[|\support{A}| \leq \frac{k}{1-\exp(-(\frac{\beta}{\beta_0})^2)} \leq k + O\left(\frac{\log n}{r^3p\beta^2}\right)\]  \end{theorem}
\begin{remark}
The theorem implies that for a small constant $c$, it suffices to have plasticity parameter 
\[
\beta \ge \frac{1}{r^2}\frac{c}{\sqrt{kp/(2\ln(n/k))}+1}.
\]
\end{remark}

\noindent Our second theorem is about {\em recall} for a single assembly, when a new stimulus from the same class is presented.  We assume that examples from an an assembly class $A$ have been presented, and a response assembly $A^*$ encoding this class has been created, by the previous theorem.  

\begin{theorem}[Recall] \label{theorem:recall}
WHP over the stimulus class, the set $C_1$ firing in response to a test assembly from the class $A$ will overlap $\assmcore{A}$ by a fraction of at least $1 - e^{-kpr}$, i.e. \[\frac{|C_1 \setminus \assmcore{A}|}{k} \leq e^{-kpr}\]
\end{theorem}
\noindent  The proof entails showing that the average weight of incoming connections to a neuron in $\assmcore{A}$ from neurons in $\stimcore{A}$ is at least \[1 + \frac{1}{\sqrt{r}}\left(\sqrt{2} + \sqrt{\frac{2}{kpr}\ln\left(\frac{n}{k}\right) + 2}\right)\]

\noindent Our third theorem is about the creation of a second assembly corresponding to a second stimulus class. This can easily be extended to many classes and assemblies.  As in the previous theorem, we assume that $O(\log k)$ examples from assembly class $A$ have been presented, and $\support{A}$ has been created.  Then we introduce $B$, a second stimulus class, with $|\stimcore{A} \cap \stimcore{B}| = \alpha k$, and present $O(\log k)$ samples to induce a series of caps, $B_1, B_2, \ldots$, with $B^*$ as their union.

\begin{theorem}[Multiple Assemblies] \label{theorem:multiple}
The total support of $B^*$ can be bounded WHP as 
\[|B^*| \leq \frac{k}{1-\exp(-(\frac{\beta}{\beta_0})^2)} \leq k + O\left(\frac{\log n}{r^3p\beta^2}\right)\]
Moreover, WHP, the overlap in the core sets $\assmcore{A}$ and $\assmcore{B}$ will preserve the overlap of the stimulus classes, so that $|\assmcore{A} \cap \assmcore{B}| \leq \alpha k$.
\end{theorem}
This time the proof relies on the fact that the average weight of incoming connections to a neuron in $\assmcore{A}$ is {\em upper-bounded} by \[\gamma \leq 1 + \frac{\sqrt{2\ln\left(\frac{n}{k}\right)} - \sqrt{2\ln((1+r)/r\alpha)}}{\alpha r \sqrt{kp}}\] 

\noindent Our fourth theorem is about classification after the creation of multiple assemblies, and shows that random stimuli from any class are mapped to their corresponding assembly. We state it here for two stimuli classes, but again it is extended to several. We assume that stimulus classes $A$ and $B$ overlap in their core sets by a fraction of $\alpha$, and that they have been projected to form a distribution of assemblies $\assmcore{A}$ and $\assmcore{B}$, respectively.

\begin{theorem}[Classification] \label{theorem:classify}
If a random stimulus chosen from a particular class (WLOG, say $B$) fires to cause a set $C_1$ of learning area neurons to fire, then WHP over the stimulus class the fraction of neurons in the cap $C_1$ and in $\assmcore{B}$ will be at least 
\[\frac{|C_1 \cap \assmcore{B}|}{k} \geq 1 - 2\exp\left(-\frac{1}{2}(\gamma \alpha - 1)^2 kpr \right)\]
where $\gamma$ is a lower bound on the average weight of incoming connections to a neuron in $\assmcore{A}$ (resp. $\assmcore{B}$) from neurons in $\stimcore{A}$ (resp. $\stimcore{B}$).
\end{theorem}


\noindent Taken together, the above results guarantee that this mechanism can learn to classify well-separated distributions, where each distribution has a constant fraction of its nonzero coordinates in a subset of $k$ input coordinates. The process is {\em naturally interpretable:} an assembly is created for each  distribution, so that random stimuli are mapped to their corresponding assemblies, and the assemblies for different distributions overlap in no more than the core subsets of their corresponding distributions.

Finally, we consider the setting where the labeling function is a linear threshold function, parameterized by an arbitrary nonnegative vector $v$ and margin $\Delta$. We will create a single assembly to represent examples on one side of the threshold, i.e. those for which $v \cdot X \ge  \|v\|_1 k / n$. We define $\mathcal D_+$ denote the distribution of these examples, where each coordinate is an independent Bernoulli variable with mean $\E(X_i) = k/n + \Delta v_i$, and define $\mathcal D_-$ to be the distribution of negative examples, where each coordinate is again an independent Bernoulli variable yet now all identically distributed with mean $k/n$. (Note that the support of the positive and negative distributions is the same; there is a small probability of drawing a positive example from the negative distribution, or vice versa.) To serve as a classifier, a fraction $1- \epsilon_+$ of neurons in the assembly must be guaranteed to fire for a positive example, and a fraction $\epsilon_- < 1 - \epsilon_+$ guaranteed \emph{not} to fire for a negative one. A test example is then classified as positive if at least a $1 - \epsilon$ fraction of neurons in the assembly fire (for $\epsilon \in [\epsilon_-, 1 - \epsilon_+]$), and negative otherwise. The last theorem shows that this can in fact be done with high probability, as long as the normal vector $v$ of the linear threshold is neither too dense nor too sparse. Additionally, we assume synapses are subject to homeostasis in between training and evaluation; that is, all of the incoming weights to a neuron are normalized to sum to 1.
\begin{theorem}[Learning Linear Thresholds]\label{theorem:halfspace}
Let $v$ be a nonnegative vector normalized to be of unit Euclidean length ($\|v\|_2 = 1$). Assume that  $\Omega(k) = \|v\|_1 \le \sqrt{n}/2$ and 
\[\Delta^2\beta \ge \sqrt{\frac{2k}{p}}(\sqrt{2\ln(n/k)+2)} + 1).\] 
Then,
sequentially presenting $\Omega(\log k)$ samples drawn at random from $\mathcal D^+$ forms an assembly $\assmcore{A}$ that correctly separates $D^+$ from $D^-$: with probability $1-o(1)$ a randomly drawn example from $\mathcal D^+$ will result in a cap which overlaps at least $3k/4$ neurons in $\assmcore{A}$, and an example from $\mathcal D^-$ will create a cap which overlaps no more than $k/4$ neurons in $\assmcore{A}$.
\end{theorem}

\begin{remark}
The bound on $\Delta^2 \beta$ leads to two regimes of particular interest: In the first, \[\beta \ge \frac{\sqrt{2\ln(n/k) + 2} + 1}{\sqrt{kp}}\] and $\Delta \ge \sqrt{k}$, which is similar to the plasticity parameter required for a fixed stimulus \citep{papadimitriou2019random} or stimulus classes; in the second, $\beta$ is a constant, and \[\Delta \ge \left(\frac{2k}{\beta^2 p}\right)^{1/4}\left(\sqrt{2\ln(n/k)+2} +1\right)^{1/2}.\]
\end{remark}

\begin{remark}
We can ensure that the number of neurons outside of $\assmcore{A}$ for a positive example or in $\assmcore{A}$ for a negative example are both $o(k)$ with small overhead\footnote{i.e. increasing the plasticity constant $\beta$ by a factor of $1 + o(1)$}, so that plasticity can be active during the classification phase. 
\end{remark}

Since our focus in this paper is on highlighting the brain-like aspects of this learning mechanism, we emphasize stimulus classes as a case of particular interest, as they are a probabilistic generalization of the single stimuli considered in \citet{papadimitriou2019random}. 
Linear threshold functions are an equally natural way to generalize a single $k$-sparse stimulus, say $v$; all the 0/1 points on the positive side of the threshold $v^\top x \ge \alpha k$ have at least an $\alpha$ fraction of the $k$ neurons of the stimulus active.

Finally, reading the output of the device by the Assembly Calculus is simple: Add a {\em readout area} to the two areas so far (stimulus and learning), and project to this area one of the assemblies formed in the learning area for each stimulus class.  The assembly in the learning area that fires in response to a test sample will cause the assembly in the readout area  corresponding to the class to fire, and this can be sensed through the {\tt readout} operation of the AC.

\paragraph{Proof overview.}
The proofs of all five theorems can be found in the Appendix.  The proofs hinge on showing that large numbers of certain neurons of interest will be included in the cap on a particular round --- or excluded from it. More specifically: \begin{itemize}
    \item To create an assembly, the sequence of caps should converge to the assembly's core set. In other words, WHP an increasing fraction of the neurons selected by the cap in a particular step will also be selected at the next one.
    \item For recall, a large fraction of the assembly should fire (i.e. be included in the cap) when presented with an example from the class.
    \item To differentiate stimuli (i.e. classify), we need to ensure that a large fraction of the correct assembly will fire, while no more than a small fraction of the other assemblies do.
\end{itemize} Following \citet{papadimitriou2019random}, we observe that if the probability of a neuron having input at least $t$ is no more than $\epsilon$, then no more than an $\epsilon$ fraction of the cohort of neurons will have input exceeding $t$ (with constant probability). By approximating the total input to a neuron as Gaussian and using well-known bounds on Gaussian tail probabilities, we can solve for $t$, which gives an explicit input threshold neurons must surpass to make a particular cap. Then, we argue that the advantage conferred by plasticity, combined with the similarity of examples from the same class, gives the neurons of interest enough of an advantage that the input to all but a small constant fraction will exceed the threshold.

\section{Experiments}
The learning algorithm has been run on both synthetic and real-world datasets, as illustrated in the figures below. Code for experiments is available at 
\url{https://github.com/mdabagia/learning-with-assemblies}.

Beyond the basic method of presenting a few examples from the same class and allowing plasticity to alter synaptic weights, the training procedure is slightly different for each of the concept classes (stimulus classes, linearly-separated, and MNIST digits). In each case, we renormalize the incoming weights of each neuron to sum to one after concluding the presentation of each class, and classification is performed on top of the learned assemblies by predicting the class corresponding to the assembly with the most neurons on. \begin{itemize}
    \item For stimulus classes, we estimate the assembly for each class as composed of the neurons which fired in response to the last training example, which in practice are the same as those most likely to fire for a random test example. 
    \item For a linear threshold, we only present positive examples, and thus only form an assembly for one class. As with stimulus classes, the neurons in the assembly can be estimated by the last training cap or by averaging over test examples. We classify by comparing against a fixed threshold, generally half the cap size.
\end{itemize}
Additionally, it is important to set the plasticity parameter ($\beta$) large enough that assemblies are reliably formed. We had success with $\beta = 0.1$ for stimulus classes and $\beta = 1.0$ for linear thresholds.


In Figure \ref{figure:experiments} (a) \& (b), we demonstrate learning of two stimulus classes, while in
Figure \ref{figure:experiments} (c) \& (d), we demonstrate the result of learning a well-separated linear threshold function with assemblies.  Both had perfect accuracy. Additionally, assemblies readily generalize to a larger number of classes (see Figure \ref{figure:fourclasses} in the appendix). We also recorded sharp threshold transitions in classification performance as the key parameters of the model are varied (see Figures \ref{fig:accuracies} \& \ref{fig:transition}).


\begin{figure}[t]
    \centering
    \includegraphics[width=0.7\linewidth]{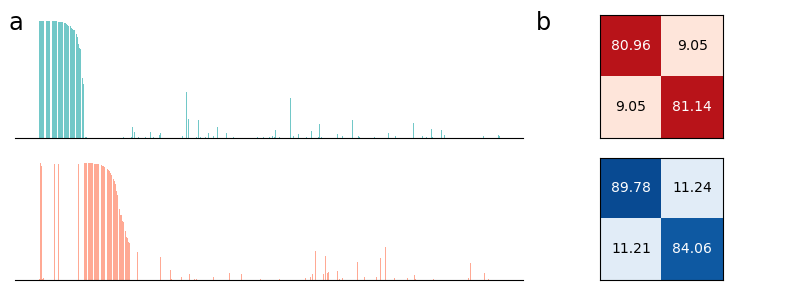}
    \includegraphics[width=0.7\linewidth]{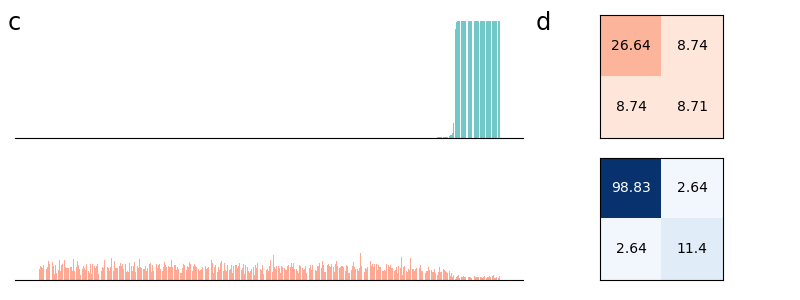}
    \caption{Assemblies learned for various concept classes. On the top two lines, we show assemblies learned for stimulus classes, and on the bottom two lines, for a linear threshold with margin. In (a) \& (c) we exhibit the distribution of firing probabilities over neurons of the learning area. In (b) \& (d) we show the average overlap of different input samples (red square) and the overlaps of the corresponding representations in the assemblies (blue square). Using a simple sum readout over assembly neurons, both stimulus classes and linear thresholds are classified with 100\% accuracy. Here, $n=10^3, k=10^2, p=0.1, r=0.9, q=0.1, \Delta = 1.0$, with 5 samples per class, and $\beta = 0.01$ (stimulus classes) and $\beta=1.0$ (linear threshold).}
    \label{figure:experiments}
\end{figure}

\begin{figure}[h]
    \centering
    \includegraphics[width=0.45\linewidth]{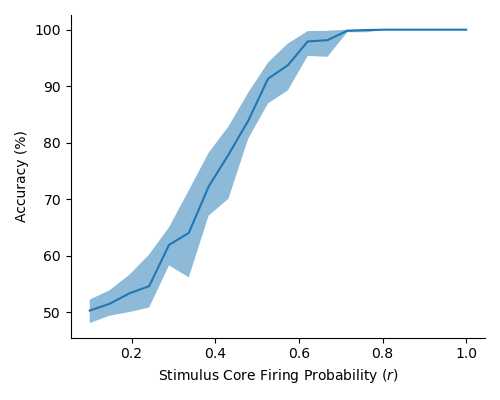}
    \includegraphics[width=0.45\linewidth]{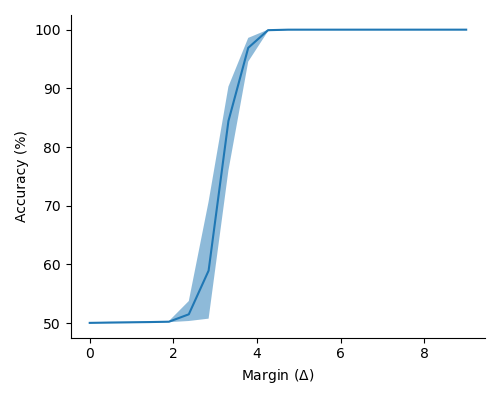}
    \caption{Mean (dark line) and range (shaded area) of classification accuracy for two stimulus classes (left) and a fixed linear threshold (right) over 20 trials, as the classes become more separable. For stimulus classes, we vary the firing probability of neurons in the stimulus core while fixing the probability for the rest at $k/n$, while for the linear threshold, we vary the margin. 
    For both we used 5 training examples with $n=1000, k=100, p=0.1$, and $\beta = 0.1$ (stimulus classes), $\beta = 1.0$ (linear threshold).}
    \label{fig:accuracies}
\end{figure}

\begin{figure}[h]
    \centering
    \includegraphics[width=0.45\linewidth]{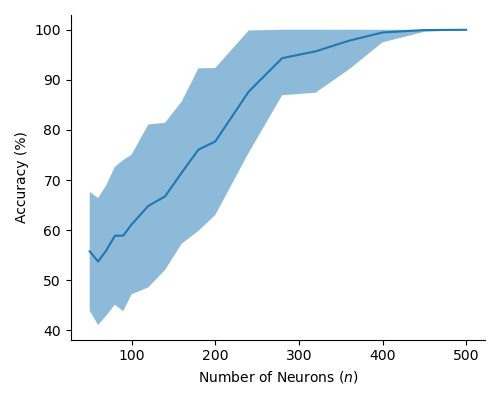}
    \includegraphics[width=0.45\linewidth]{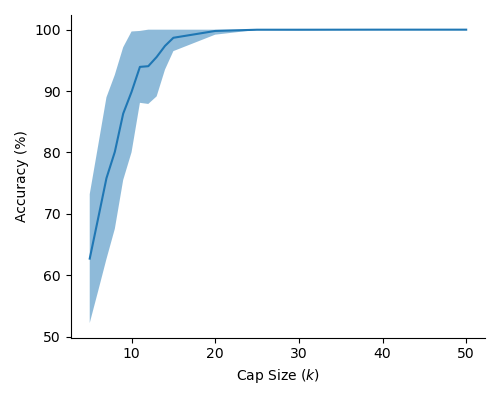}
    \caption{Mean (dark line) and range (shaded area) of classification accuracy of two stimulus classes for various values of the number of neurons ($n$, left) and the cap size ($k$, right). For variable $n$, we let $k = n / 10$; for variable $k$, we fix $n = 1000$. Other parameters are fixed, as $p = 0.1, r = 0.9, q = k/n$, and $\beta = 0.1$.}
    \label{fig:transition}
\end{figure}

There are a number of possible extensions to the simplest strategy, where within a single brain region we learn an assembly for each concept class and classify based on which assembly is most activated in response to an example. 
We compared the performance of various classification models on MNIST as the number of features increases. The high-level model is to extract a certain number of features using one of the five different methods, and then find the best linear classifier (of the training data) on these features to measure performance (on the test data). The five different feature extractors are:
\begin{itemize}
    \item Linear features. Each feature's weights are sampled i.i.d. from a Gaussian with standard deviation $0.1$.
    \item Nonlinear features. Each feature is a binary neuron: it has $784$ i.i.d. Bernoulli$(0.2)$ weights, and `fires' (has output $1$, otherwise $0$) if its total input exceeds the expected input ($70 \times 0.2$).
    \item Large area assembly features. In a single brain area of size $m$ with cap size $m / 10$, we attempt to form an assembly for each class. The area sees a sequence of $5$ examples from each class, with homeostasis applied after each class. Weights are updated according to Hebbian plasticity with $\beta = 1.0$. Additionally, we apply a negative bias: A neuron which has fired for a given class is heavily penalized against firing for subsequent classes.
    \item 'Random' assembly features. For a total of $m$ features, we create $m / 100$ different areas of $100$ neurons each, with cap size $10$. We then repeat the large area training procedure above in each area, with the order of the presentation of classes randomized for each area.
    \item 'Split' assembly features: For a total of $m$ features, we create $10$ different areas of $m / 10$ neurons each, with cap size $m / 100$. Area $i$ sees a sequence of $5$ examples from class $i$. Weights are updated according to Hebbian plasticity, and homeostasis is applied after training.
\end{itemize}
After extracting features, we train the linear classification layer to minimize cross-entropy loss on the standard MNIST training set ($60000$ images) and finally test on the full test set ($10000$ images). 

\begin{figure}[b!]\label{fig:mnistcompare}
    \centering
    \includegraphics[width=0.7\linewidth]{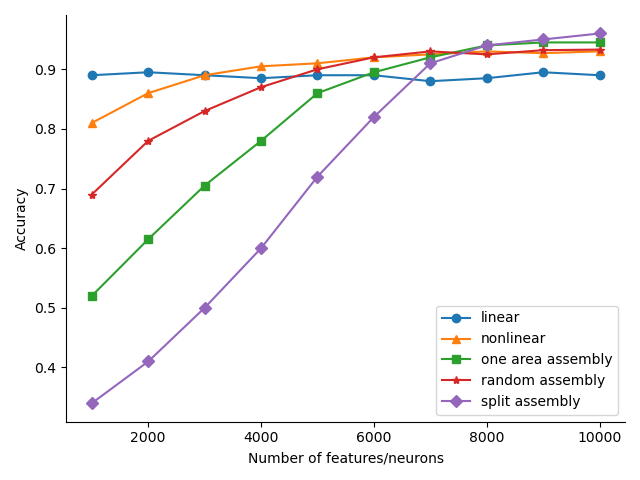}
    \caption{MNIST test accuracy as the number of features increases, for various classification models. 'Split' assembly features, which forms an assembly for class $i$ in area $i$, achieves the highest accuracy with the largest number of features.}
\end{figure}

The results as the total number of features ranges from $1000$ to $10000$ is shown in Fig. \ref{fig:mnistcompare}. 'Split' assembly features are ultimately the best of the five, with 'split' features achieving $96\%$ accuracy with $10000$ features. However, nonlinear features outperform 'split' and large-area features and match 'random' assembly features when the number of features is less than $8000$. 
For reference, the linear classifier gets to $89\%$, while a two-layer neural network with width $800$ trained end-to-end gets to $98.4\%$.

Going further, one could even create a hierarchy of brain areas, so that the areas in the first ``layer'' all project to a higher-level area, in hopes of forming assemblies for each digit in the higher-level area which are more robust. In this paper, our goal was to highlight the potential to form useful representations of a classification dataset using assemblies, and so we concentrated on a single layer of brain areas with a very simple classification layer on top. It will be interesting to explore what is possible with more complex architectures. 

\section{Discussion}
Assemblies are widely believed to be involved in cognitive phenomena, and the AC provides evidence of their computational aptitude.  Here we have made the first steps towards understanding how {\em learning} can happen in assemblies. Normally, an assembly is associated with a stimulus, such as Grandma. We have shown that this can be extended to {\em a distribution over stimuli.} Furthermore, for a wide range of model parameters, distinct assemblies can be formed for multiple stimulus classes in a single brain area, so long as the classes are reasonably differentiated.

A model of the brain at this level of abstraction should allow for the kind of classification that the brain does effortlessly --- e.g., the mechanism that enables us to understand that individual frames in a video of an object depict the same object. With this in mind, the learning algorithm we present is remarkably parsimonious: it generalizes from a handful of examples which are seen only once, and requires no outside control or supervision other than ensuring multiple samples from the same concept class are presented in succession (and this latter requirement could be relaxed in a more complex architecture which channels stimuli from different classes).  Finally, even though our results are framed within the Assembly Calculus and the underlying brain model, we note that they have implications far beyond this realm. In particular, they suggest that \emph{any} recurrent neural network, equipped with the mechanisms of plasticity and inhibition, will naturally form an assembly-like group of neurons to represent similar patterns of stimuli.

But of course, many questions remain.  In this first step we considered a single brain area --- whereas it is known that assemblies draw their computational power from the interaction, through the AC, among many areas.  
We believe that a more general architecture encompassing a hierarchy of interconnected brain areas, where the assemblies in one area act like stimulus classes for others, can succeed in learning more complex tasks --- and even within a single brain area improvements can result from optimizing the various parameters, something that we have not tried yet.   

In another direction, here we only considered Hebbian plasticity, the simplest and most well-understood mechanism for synaptic changes. Evidence is mounting in experimental neuroscience that the range of plasticity mechanisms is far more diverse \citep{magee2020synaptic}, and in fact it has been demonstrated recently \citep{payeur2021burst} that more complex rules are sufficient to learn harder tasks. Which plasticity rules make learning by assemblies more powerful?

We showed that assemblies can learn nonnegative linear threshold functions with sufficiently large margins. Experimental results suggest that the requirement of nonnegativity is a limitation of our proof technique, as empirically assemblies readily learn arbitrary linear threshold functions (with margin). What other concept classes can assemblies provably learn?  We know from support vector machines that linear threshold functions can be the basis of far more sophisticated learning when their input is pre-processed in specific ways, while the celebrated results of \citet{rahimi2007random} demonstrated that certain families of random nonlinear features can approximate sophisticated kernels quite well. What would constitute {\em a kernel} in the context of assemblies? The sensory areas of the cortex (of which the visual cortex is the best studied example) do pre-process sensory inputs extracting features such as edges, colors, and motions. Presumably learning by the non-sensory brain --- which is our focus here --- operates on the output of such pre-processing. We believe that studying the implementation of kernels in cortex is a very promising direction for discovering powerful learning mechanisms in the brain based on assemblies.

\acks{We thank Shivam Garg, Chris Jung, and Mirabel Reid for helpful discussions. MD is supported by an NSF Graduate Research Fellowship. SV is supported in part by NSF awards CCF-1909756, CCF-2007443 and CCF-2134105. CP is supported by NSF Awards CCF-1763970 and CCF-1910700, and by a research contract with Softbank.}

\bibliography{references}

\newpage
\appendix
\section*{Appendix: Further experimental results}\label{appendix:experiments}

\begin{figure}[h]
    \centering
    \includegraphics[width=0.8\linewidth]{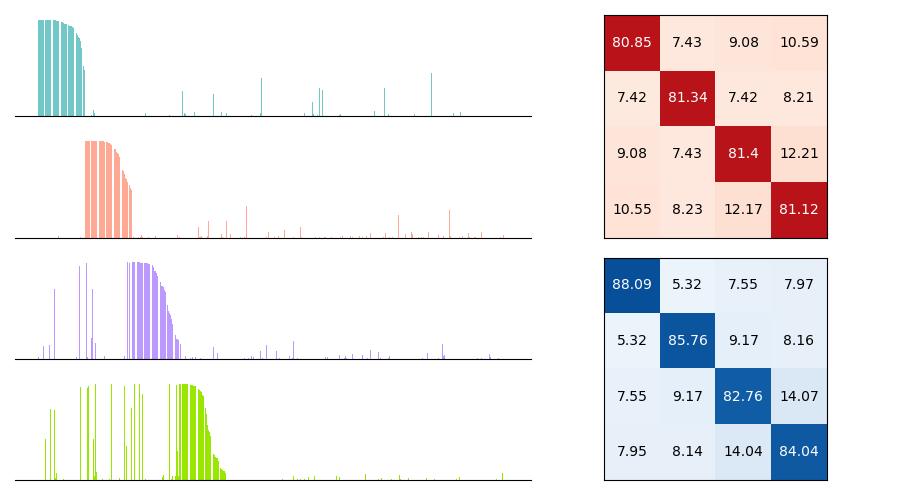}
    \caption{Assemblies learned for four stimulus classes. On the left, the distributions of firing probabilities over neurons; on the right, the average overlap of the assemblies. Each additional class overlaps with previous ones, yet a simple readout over assembly neurons allows for perfect classification accuracy.  Here, $n=10^3, k=10^2, p=0.1, r=0.9, q=0.1, \beta=0.1$, with 5 samples per class.}
    \label{figure:fourclasses}
\end{figure}

\begin{figure}[h]
    \centering
    \includegraphics[width=\linewidth]{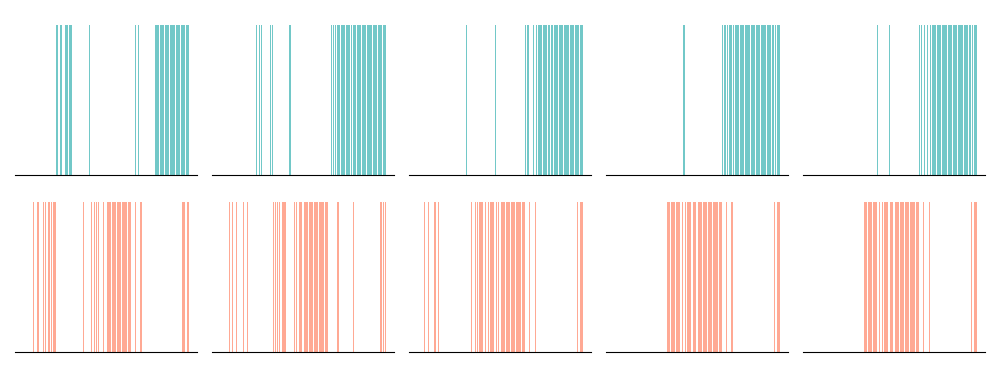}
    \caption{Assemblies formed during training. Each row is the response of the neural population to examples from the respective class. At each step, a new example from the appropriate class is presented. 
    Due to plasticity, a core set emerges for the class after only a few rounds. (Here, five rounds are shown.) Inputs are drawn from two stimulus classes, with $n=10^3, k=10^2, p=0.1, r=0.9, q=0.1$ and $\beta=0.1$.}
    \label{figure:fiverounds}
\end{figure}


\section*{Appendix: Proofs}

\subsection*{Preliminaries}
We will need a few lemmas. The first is the well-known Berry-Esseen theorem:

\begin{lemma} \label{lemma:be}
Let $X_1, \ldots, X_n$ be independent random variables with $\E[X_i] = 0, \E[X_i^2] = \sigma_i^2$ and $\E[|X_i|^3] = \rho_i < \infty$. Let
\[S = \left(\sum_{i=1}^n \sigma_i^2\right)^{-1/2} \left(\sum_{i=1}^n X_i \right)\]
Denote by $F_S$ the CDF of $S$, and $\Phi$ the standard normal CDF.
There exists some constant $C$ such that
\[\sup_{x \in \R} |F(x) - \Phi(x)| \le C\left(\sum_{i=1}^n \sigma_i^2\right)^{-1/2} \max_i \frac{\rho_i}{\sigma_i^2} \]
\end{lemma}

\noindent This implies the following:

\begin{lemma} \label{lemma:normalapprox}
Let $X_1, \ldots, X_{2n}$ denote the weights of the edges incoming to a neuron in the brain area from its neighbors (i.e. $X_i = w_i$ w.p. $p$, and $0$ otherwise). Denote their sum as $S = \sum_i X_i$, and consider the normal random variable
\[Y \sim \mathcal N(\E[S], \var[S])\]
Then for $p = o(1)$, and $w_i = o(\sqrt{np})$,
\[\sup_{x \in \R} |\Pr(S < t) - \Pr(Y < t)| = o(1)  \]
\end{lemma}

\begin{proof}
Let $\tilde X_i = X_i - \E[X_i]$. We have
\[\E[\tilde X_i^2] = p(1-p)^2w_i^2 + (1-p)p^2w_i^2 = p(1-p)w_i^2\]
and
\[\E[|\tilde X_i|^3] = p(1-p)^3w_i^3 + (1-p)p^3w_i^3 \le p(1-p)w_i^3\]
Let 
\[\tilde S = \frac{S - \E[S]}{\sqrt{\var[S]}} = \left(\sum_{i=1}^{2n} \E[\tilde X_i^2]\right)^{-1/2} \left(\sum_{i=1}^{2n} \tilde X_i \right)\]
Then by Lemma \ref{lemma:be}, there exists some constant $C$ so that
\[\sup_{x \in \R} |F_{\tilde S}(x) - \Phi(x)| \le C \left(\sum_{i=1}^{2n} \E[\tilde X_i^2]\right)^{-1/2} \max_i \frac{\E[|\tilde X_i|^3]}{\E[\tilde X_i^2]}\]
As $1 \le w_i =o(\sqrt{np})$, it follows that
\[\sum_{i=1}^{2n} \E[\tilde X_i^2] = p(1-p) \sum_{i=1}^{2n} w_i^2 \ge 2p(1-p) n\]
and
\[\frac{\E[|\tilde X_i|^3]}{\E[\tilde X_i^2]} \le w_i = o(\sqrt{np})\]
Hence, 
\[\sup_{x \in \R} |F_{\tilde S}(x) - \Phi(x)| \le C \frac{w_i}{\sqrt{2p(1-p)n}} = o\left(\frac{1}{\sqrt{1-p}}\right) = o(1)\]
for $p = o(1)$. Noting that 
\[\Pr(S < t) = \Pr\left(\tilde S < \frac{t - \E[S]}{\sqrt{\var[S}}\right) = F_{\tilde S}\left(\frac{t - \E[S]}{\sqrt{\var[S]}}\right)\]
and using the affine property of normal random variables completes the proof.
\end{proof}

\noindent In particular, we may substitute appropriate Gaussian tail bounds (such as the one provided in the following lemma) for tail bounds on sums of independent weighted Bernoullis throughout.

\begin{lemma}\label{lemma:tail}
For $X \sim \mathcal N(0, 1)$ and $t > 0$, \[\Pr (X > t) \leq \frac{1}{\sqrt{2\pi}t} e^{-t^2/2}\]
For $X \sim \mathcal N(\mu, \sigma^2)$ and $P(X > t) = p$, we have \[t = \mu + \sigma \sqrt{2\ln(1/p) + \ln(2\ln(1/p))} + o(1)\]
\end{lemma}
\begin{proof}
Recall that \[\Pr(X > t) = \int_t^\infty \frac{1}{\sqrt{2\pi}} e^{-\tau^2 / 2} \text{d}\tau\] Then observing that for $\tau \ge t$, \[ e^{-\tau^2 / 2} \le \frac{\tau}{t} e^{-\tau^2 / 2}\] we have \[\Pr(X > t) \le \frac{1}{\sqrt{2\pi}t}\int_t^\infty \tau e^{-\tau^2 / 2} = \frac{1}{\sqrt{2\pi}t}e^{-t^2/2} \] For the second part, simply solve for $t$.
\end{proof}

\noindent Next, we will use the distribution of a normal random variable, conditioned on its sum with another normal random variable. 

\begin{lemma}\label{lemma:bayes}
For $X \sim \mathcal N(\mu_x, \sigma_x^2)$, $Y \sim \mathcal N(\mu_y, \sigma_y^2)$, and $Z = X + Y$, then conditioning $X$ on $Z$ gives \[X |(Z = z) \sim \mathcal N\left(\frac{\sigma_x^2}{\sigma_x^2 + \sigma_y^2}z + \frac{\sigma_y^2\mu_x- \sigma_x^2\mu_y}{\sigma_x^2 + \sigma_y^2}, \frac{\sigma_x^2\sigma_y^2}{\sigma_x^2 + \sigma_y^2}\right)\]
\end{lemma}

\begin{proof}
Bayes' theorem provides \[f_{X|Z=z}(x, z) = \frac{f_{Z | X=x}(z, x)f_X(x)}{f_Z(z)}\] where $f_W$ is the probability density function for variable $W$. From the definition, $f_{Z | X = x}(z, x) = f_Y(z - x)$. Then substituting the Gaussian probability density function and simplifying, we have: \begin{align*}
    f_{X|Z=z}(x, z) &= \frac{\frac{1}{\sqrt{2\pi\sigma_y^2}}\exp\left(-\frac{(z-x - \mu_y)^2}{2\sigma_y^2}\right) \frac{1}{\sqrt{2\pi\sigma_x^2}} \exp\left(-\frac{(x-\mu_x)^2}{2\sigma_x^2}\right)}{\frac{1}{\sqrt{2\pi(\sigma_x^2 + \sigma_y^2)}} \exp\left(-\frac{(z-(\mu_x + \mu_y))^2}{2(\sigma_x^2 + \sigma_y^2)}\right)}\\
    &= \frac{1}{\sqrt{2\pi \frac{\sigma_x^2\sigma_y^2}{\sigma_x^2 + \sigma_y^2}}} \exp\left(-\frac{1}{2\frac{\sigma_x^2\sigma_y^2}{\sigma_x^2+\sigma_y^2}}\left(x - \left(\frac{\sigma_x^2}{\sigma_x^2 + \sigma_y^2}z + \frac{\sigma_y^2\mu_x- \sigma_x^2\mu_y}{\sigma_x^2 + \sigma_y^2}\right)\right)^2\right)\\
\end{align*}
\end{proof}

\noindent The following is the distribution of a binomial variable $X$, given that we know the value of another binomial variable $Y$ which uses $X$ as its number of trials.

\begin{lemma} \label{lemma:coinflip}
Denote by $\mathcal B(n, p)$ the binomial distribution over $n$ trials with probability of success $p$. Let $X \sum \mathcal B(n, p)$ and $Y | X \sim \mathcal B(X, q)$. Then
\[X | Y \sim Y + \mathcal B(n - Y, \tfrac{p(1-q)}{1-pq})\]
\end{lemma}
\begin{proof}
Via Bayes' rule,
\[\Pr(X=x | Y = y) = \frac{\Pr(Y = y | X = x) \Pr(X = x)}{\Pr(Y = y)}\]
It is well-known that $Y \sim \mathcal B(n, pq)$. Hence, using the formulae for the distributions and simplifying,
\begin{align*}
    \Pr(X = x | Y = y) &= \frac{{x \choose y} q^y (1-q)^{x-y} {n \choose x} p^x (1-p)^{n-x}}{{n \choose y} (pq)^y (1-pq)^{n-y}}\\
    &= {n-y \choose x-y} \frac{(p(1-q))^{x-y} (1-p)^{n-x}}{(1-pq)^{n-y}}\\
    &= {n-y \choose x-y} \left(\frac{p(1-q)}{1-pq}\right)^{x-y} \left(\frac{1-p}{1-pq}\right)^{n-x}\\
    &= {n-y \choose x-y} \left(\frac{p(1-q)}{1-pq}\right)^{x-y} \left(1 - \frac{p(1-q)}{1-pq}\right)^{n-x}
\end{align*} Note that for $Z \sim \mathcal B(n - y, \tfrac{p(1-q)}{1-pq})$, we have
\[\Pr(X = x | Y = y) = \Pr(Z = x-y)\]
and so $X | Y \sim Y + Z$. 
\end{proof}

\noindent The next observation is useful: Exponentiating a random variable by a base close to one will increase its concentration.
\begin{lemma} \label{lemma:lognormal}
Let $X \sim \mathcal N(\mu_x, \sigma_x^2)$ be a normal variable, and let $Y = (1 + \beta)^X$. Then $Y$ is lognormal with \begin{align*}
    \E(Y) &= (1+\beta)^{\mu_x}(1+\beta)^{\ln(1+\beta)\sigma_x^2/2}\\
    \var{Y} &= ((1+\beta)^{\ln(1+\beta)\sigma_x^2} - 1)(1+\beta)^{2\mu_x}(1+\beta)^{\ln(1+\beta)\sigma_x^2}
\end{align*} 
\end{lemma}
In particular, for $\ln(1+\beta)\sigma_x^2$ close to $0$, $Y$ is highly concentrated at $(1+\beta)^{\mu_x}$.
\begin{proof}
Observe that $Y = e^{\ln(1+\beta)X}$, so it is clearly lognormal. So, define $\tilde X = \ln(1+\beta)X \sim \mathcal N(\ln(1+\beta)\mu_x, \ln(1+\beta)^2\sigma_x^2)$. Then we have \begin{align*}
    \E(Y) &= \exp\left(\E(\tilde X) + \frac{1}{2}\var{\tilde X}\right)\\
    &= \exp\left(\ln(1+\beta)\mu_x + \frac{1}{2}\ln(1+\beta)^2\sigma_x^2\right)\\
    &= (1+\beta)^{\mu_x}(1+\beta)^{\ln(1+\beta)\sigma_x^2/2}\\
\end{align*} and \begin{align*}
    \var{Y} &= \left(\exp\left(\var{\tilde X}\right) - 1\right)\exp\left(2\E(\tilde X) + \var{\tilde X}\right)\\
    &= \left(\exp\left(\ln(1+\beta)^2\sigma_x^2\right) - 1\right)\exp\left(2\ln(1+\beta)\mu_x + \ln(1+\beta)^2\sigma_x^2\right)\\
    &= ((1+\beta)^{\ln(1+\beta)\sigma_x^2} - 1)(1+\beta)^{2\mu_x}(1+\beta)^{\ln(1+\beta)\sigma_x^2}
\end{align*} Furthermore, if $\ln(1+\beta)\sigma_x^2 \approx 0$, then $(1+\beta)^{\ln(1+\beta)\sigma_x^2} \approx 1$ and we obtain the concentration.
\end{proof}

\noindent For learning a linear threshold function with an assembly (Theorem \ref{theorem:halfspace}), we will require an additional lemma. 
\begin{lemma}\label{lemma:dotprod}
Let $X_1, \ldots, X_n$ and $Y_1, \ldots, Y_n$ be independent Bernoulli variables, and let $Z = \sum_{j=1}^n X_i Y_i$. Then for any $t \ge 0$, \[ \E(Y_i \vert Z \ge \E(Z) + t) \ge \E(Y_i) \]
\end{lemma}
\begin{proof}
Bayes' rule gives \[\E(Y_i \vert Z \ge \E(Z) + t) = \Pr(Y_i=1 \vert Z \ge \E(Z) + t) = \frac{\Pr(Z \ge \E(Z) + t \vert Y_i = 1) \Pr(Y_i = 1)}{\Pr(Z \ge \E(Z) + t)}\] Then observe that the events $Z \ge \E Z + t$ and $Y_i = 1$ are positively correlated, and so 
\[\Pr(Z \ge \E Z + t \vert Y_i = 1) \ge \Pr(Z \ge \E Z + t)\]
Then substituting gives
\begin{align*}
    \E(Y_i \vert Z \ge \E(Z) + t) &\ge \frac{\Pr\left(Z \ge \E(Z) + t\right) \Pr(Y_i = 1)}{\Pr(Z \ge \E(Z) + t)}\\
    &= \E(Y_i)
\end{align*} as required.
\end{proof}

\noindent Lastly, the following lemma allows us to translate a bound on the weight between certain synapses into a bound on the number of rounds (or samples) required.

\begin{lemma} \label{lemma:nrounds}
Consider a neuron $i$, connected by a synapse to a neuron $j$ with weight initially 1, and equipped with a plasticity parameter $\beta$. Assume that $j$ fires with probability $p$ and $i$ fires with probability $q$ on each round, and that there are at least $T$ rounds, with \[T \geq \frac{1}{pq}\frac{\ln \gamma}{\ln(1+\beta)}\] Then the synapse will have weight at least $\gamma$ in expectation.
\end{lemma}

We are now equipped to prove the theorems.

\subsection*{Proof of Theorem \ref{theorem:creation}} \label{appendix:creation}
Let $\mu_t$ be the fraction of first-timers in the cap on round $t$. The process stabilizes when $\mu_t < 1/k$, as then no new neurons have entered the cap. 

For a given neuron $i$, let $X(t)$ and $Y(t)$ denote the input from connections to the $k$ neurons in $\stimcore{A}$ and the $n-k$ neurons outside of $\stimcore{A}$, respectively, on round $t$. For a neuron which has never fired before, they are distributed approximately as \[ X(t) \sim \mathcal N(kpr, kpr) \quad Y(t) \sim \mathcal N(kpq, kpq) \] which follows from Lemma \ref{lemma:normalapprox},
for a total input of $X(t) + Y(t) \sim \mathcal N(kpr + kpq, kpr + kpq)$. (Note that we ignore small second-order terms in the variance.) To determine which neurons will make the cap on the first round, we need a threshold that roughly $k$ of $n$ draws from $X(1) + Y(1)$ will exceed, with constant probability. In other words, we need the probability that $X(1) + Y(1)$ exceeds this threshold to be about $k/n$. Taking $L = 2\ln(n/k)$ and using the tail bound in Lemma \ref{lemma:tail}, we find the threshold for the first cap to be at least \[C_1 = kp(r + q) + \sqrt{kp(r + q)L}\]

On subsequent rounds, there is additional input from connections to the previous cap, distributed as $\mathcal N(kp, kp(1-p))$. Using $\mu_t$ as the fraction of first-timers, a first-time neuron must be in the top $\mu_tk$ of the $n-k \sim n$ neurons left out of the previous cap. The activation threshold is thus \[C_t = kp(1+r) + kpq + \sqrt{kp(1+r + q)(L + 2\ln(1/\mu_t))}\]

Now consider a neuron $i$ which fired on the first round. We know that $X(1) + Y(1) \ge C_1$, so using Lemma \ref{lemma:bayes}, \[X(1) | (X(1) + Y(1) = C_1) \sim \mathcal N\left(\frac{r}{r + q}C_1, kp\frac{rq}{r+q} \right)\] If $X(1) = x$, Lemma \ref{lemma:coinflip} indicates that the true number of connections with stimulus neurons is distributed roughly as $\tilde X | (X(1) = x) \sim \mathcal N(x + (k - x)p(1-r), (k - x)p(1-r))$. Conditioning on $X(1) + Y(1) = C_1$, ignoring second-order terms, and bounding the variance as $kp$, we have \[\tilde X | (X(1) + Y(1) = C_1) \sim \mathcal N\left(kp(1-r) + \frac{r}{r + q}C_1, kp\right)\]

On the second round, the synapses between neuron $i$ and stimulus neurons which fired have had their weights increased by a factor of $1 + \beta$, and these stimulus neurons will fire on the second round with probability $r$. An additional $\tilde X - X(1)$ stimulus neurons have a chance to fire for the first time. Neuron $i$ also receives recurrent input from the $k$ other neurons which fired the previous round, which it is connected to with probability $p$. So, the total input to neuron $i$ is roughly \[\mathcal N\left((1 + \beta)\frac{r^2}{r+q}C_1 + kpr(1-r), kp\left(1 + \frac{rq}{r+q}\right)\right) + \mathcal N(kp(1+q), kp(1+q))\]

In order for $i$ to make the second cap, we need that its input exceeds the threshold for first-timers, i.e. \[(1+\beta)\frac{r^2}{r+q}C_1 + kp(1+ r(1-r)+q) + Z \geq C_2\] where $Z \sim \mathcal N\left(0, kp(1 + r + q)\right)$. Taking $\mu = \mu_2$, we have the following: \begin{align*}
    \Pr (i \in C_2 &| i \in C_1) = 1 - \mu\\
    &\geq \Pr \left(Z \geq C_2 - (1+\beta)\frac{r^2}{r+q}C_1 - kp(1+ r(1-r)+q)\right)\\
    &\geq \Pr \left(Z \geq -\beta kpr^2 - (1+\beta)\frac{r^2}{\sqrt{r+q}}\sqrt{kpL} + \sqrt{kp(1+r + q)(L + 2\ln(1/\mu))}\right)
\end{align*}
Now, normalizing $Z$ to $\mathcal N(0, 1)$ we have (again by the tail bound) \[1 - \mu \geq 1 - \exp\left(-\frac{\left(\beta \sqrt{kp}r^2 + (1+\beta)\frac{r^2}{\sqrt{r+q}}\sqrt{L} + \sqrt{(1+r + q)(L + 2\ln(1/\mu))}\right)^2}{2(1+r+q)}\right)\] 
More clearly, this means \[\sqrt{2(1+r+q)\ln(1/\mu)} \leq \beta \sqrt{kp}r^2 + (1+\beta)\frac{r^2}{\sqrt{r+q}}\sqrt{L} + \sqrt{(1+r + q)(L + 2\ln(1/\mu))}\] 
Then taking \[\beta \geq \beta_0 = \frac{\sqrt{r+q}}{r^2}\frac{\left(\sqrt{1+r+q} - \frac{r^2}{\sqrt{r+q}}\right)\sqrt{L} + \sqrt{2(1 + r+q)}}{\sqrt{kp} + \sqrt{L}}\] gives $\mu \leq 1/e$, i.e. the overlap between the first two caps is at least a $1 - 1/e$ fraction.

Now, we seek to show that the probability of a neuron leaving the cap drops off exponentially the more rounds it makes it in. Suppose that neuron $i$ makes it into the first cap and stays for $t$ consecutive caps. Each of its connections with stimulus neurons will be strengthened by the number of times that stimulus neuron fired, roughly $\mathcal N(tr, tr(1-r))$ times. Using Lemma \ref{lemma:lognormal}, the weight of the connection with a stimulus neuron is highly concentrated around $(1+\beta)^{tr}$. Furthermore we know that $i$ has at least $\tilde X | (X(1) + Y(1) = C_1) \sim \mathcal N\left(kp(1-r) + \frac{r}{r + q}C_1, kp\right)$ such connections, of which $\mathcal N\left(kpr(1-r) + \frac{r^2}{r + q}C_1, kpr\right)$ will fire. So, the input to neuron $i$ will be at least \[(1+\beta)^{tr}(kpr(1-r) + \frac{r^2}{r + q}C_1) + kp(1+q) + Z\] where $Z \sim \mathcal N(0, kp(1 + (1+\beta)^{2tr}r+q))$

To stay in the $(t+1)$th cap, it suffices that this input is greater than $C_{t+1}$, the threshold for first-timers. Using $\mu = \mu_{t+1}$ and reasoning as before: \begin{align*}
    \Pr (i \in C_{t+1} &| i \in C_1 \cap \ldots \cap C_t) = 1 - \mu\\
    &\geq \Pr \left(Z > C_{t+1} - (1+\beta)^{tr}\left(kpr(1-r) + \frac{r^2}{r + q}C_1\right) - kp(1+q)\right)\\
    &= \Pr \left(Z > -t\beta kpr - (1+tr\beta)\frac{r^2}{\sqrt{r+q}}\sqrt{kpL} + \sqrt{kp(1+r + q)(L + 2\ln(1/\mu))}\right)\\
    &\geq 1 - \exp\left(-\frac{\left(tr\beta \sqrt{kp} + (1+tr\beta)\frac{r^2}{\sqrt{r+q}}\sqrt{L} - \sqrt{(1+r + q)(L + 2\ln(1/\mu))}\right)^2}{2(1+r+q)}\right)
\end{align*} where in the last step we approximately normalized $Z$ to $\mathcal N(0, 1)$.
Then \[\beta \geq \frac{1}{tr^2}\frac{\sqrt{(1+r+q)(L + 2t^2)} - r^2)\sqrt{L} + t\sqrt{2}}{\sqrt{kp} + \sqrt{L}}\] will ensure $\mu \leq e^{-t^2}$, which is no more than $\beta_0$.

Now, let neuron $i$ be a first time winner on round $t$. Let $X \sim \mathcal N(kpr, kpr)$ denote the input from stimulus neurons, $Y \sim \mathcal N(kp, kp)$ the input from recurrent connections to neurons in the previous cap, and $Z \sim \mathcal N(kpq, kpq)$ the input from nonstimulus neurons. Then conditioned on $X + Y + Z = C_t$, the second lemma indicates that \begin{align*}
    X | (X + Y + Z = C_t) &\sim \mathcal N\left(\frac{r}{1 + r + q}C_t, kp\frac{r(1+q)}{1 + r + q}\right)\\
    Y | (X + Y + Z = C_t) &\sim \mathcal N\left(\frac{1}{1 + r + q}C_t, kp\frac{r+q}{1 + r+q}\right)
\end{align*} So, the input on round $t+1$ is at least 
\[(1+\beta)\frac{1 - \mu_t + r^2}{1 + r + q}C_t + kpr(1-r) + kp\mu_t + kpq + Z\] where $Z \sim \mathcal N\left(0, kp\left((1+\beta)^2\frac{r^2(1+q) + (1-\mu_t)(r+q)}{1 + r + q} + r(1-r) + q\right)\right)$. By the usual argument we have \begin{align*}
    \Pr &(i \in C_{t+1} | i \in C_t) = 1 - \mu_{t+1}\\
    &\geq \Pr \left(Z \geq C_{t+1} - (1+\beta)\frac{1 - \mu_t + r^2}{1 + r + q}C_t - kpr(1-r) - kp\mu_t - kpq\right)
\end{align*} 
So, we will have $\mu_{t+1} < e^{-1}\mu_t $ when \[\beta \geq \frac{1}{1 - \mu_t + r^2}\frac{\frac{r(1-r) + q + \mu_t}{\sqrt{1 + r + q}}\sqrt{L + 2\ln(1/\mu_t)} + \sqrt{2\ln(1/\mu_t)}}{\sqrt{kp} + \sqrt{\frac{L + 2\ln(1 / \mu_t)}{1 + r + q}}}\] which is smaller than $\beta_0$. 
Assuming $r+q \sim 1$, we may simplify $\beta_0$, so that \[\beta_0 = \frac{1}{r^2}\frac{\left(\sqrt{2} - r^2\right)\sqrt{L} + \sqrt{6}}{\sqrt{kp} + \sqrt{L}}\]

So, if $\beta \geq \beta_0$, the probability of leaving the cap once in the cap $t$ times drops off exponentially. We can conclude that no more than $\ln(k)$ rounds will be required for convergence. Additionally, assuming that a neuron enters the cap at time $t$, let $1 - p_\tau$ denote the probability it leaves after $\tau$ rounds. Then its probability of staying in the cap on all subsequent rounds is \[\prod_{\tau \geq 1} p_\tau \geq \prod_{\tau \geq 1}\left(1 - \exp\left(-\tau^2\left(\frac{\beta}{\beta_0}\right)^2\right)\right) \geq 1 - \exp\left(-\left(\frac{\beta}{\beta_0}\right)^2\right)\] Thus, every neuron that makes it into the cap has a probability at least $1 - \exp(-(\beta/\beta_0)^2))$ of making every subsequent cap, so the total support of all caps together is no more than $ k / (1 - \exp(-(\beta/\beta_0)^2)))$ in expectation.
\hfill$\blacksquare$

\subsection*{Proof of Theorem \ref{theorem:recall}} \label{appendix:recall}
Let $\mu$ denote the fraction of newcomers in the cap. A neuron in $\assmcore{A}$ can expect an input of \[X_a = \gamma kpr + kpq + Z_a\] where $Z_a \sim \mathcal N(0, \gamma^2kpr + kpq)$, while neurons outside of $\assmcore{A}$ can expect an input of \[X = kpr + kpq + Z\] where $Z \sim \mathcal N(0, kpr + kpq)$. Then the threshold is roughly \[C_1 = kpr + kpq + \sqrt{kp(r+q)(L + 2\ln(1/\mu))}\] For a neuron $i$ in $\assmcore{A}$ to make the cap, it needs to exceed this threshold. We have \begin{align*}
    \Pr (i \in C_1 | i \in \assmcore{A}) &= 1- \mu\\
    &\geq \Pr (X_a \geq C_1)\\
    &= \Pr \left(Z_a \geq -(\gamma-1) kpr + \sqrt{kp(r+q)(L + 2\ln(1/\mu))}\right)
\end{align*} Applying the tail bound gives \[\sqrt{2\ln(1/\mu)} \leq (\gamma - 1) \frac{r}{\sqrt{r+q}} \sqrt{kp} - \sqrt{L + 2\ln(1/\mu)} \] so for $r+q \sim 1$ and 
\[\gamma \geq 1 + \frac{1}{\sqrt{r}}\left(\sqrt{2} + \sqrt{L/kpr + 2}\right)\]
we will have $\mu \leq e^{-kpr}$.
\hfill$\blacksquare$

\subsection*{Proof of Theorem \ref{theorem:multiple}} \label{appendix:multiple}

Let $\nu_t$ be the fraction of neurons in $A^*$ included in the cap on round $t$, and let $\mu_t$ be the fraction of true first-timers. The input on the first round for first-timers will be $\mathcal N(kpr, kpr) + \mathcal N(kpq, kpq)$, while for neurons in $\assmcore{A}$ will be 
\[\mathcal N(\gamma \alpha kpr, \gamma^2 \alpha kpr) + \mathcal N((1-\alpha)kpr, (1-\alpha)kpr) + \mathcal N(kpq, kpq)\] 
For a neuron not in $\assmcore{A}$ to make the cap, it needs to be in the top $(1 - \nu_1)k$ of $n - k \sim n$ draws. Thus, the threshold is at least \[C_1 = kp(r+q) + \sqrt{kp(r+q)(L - 2\ln(1-\nu_1))}\]
For a neuron in $\assmcore{A}$ to make the cap, it needs to exceed this threshold. Thus, we have \begin{align*}
    \Pr (i \in C_1 | i \in \assmcore{A}) &= \nu_1\\
    &\leq \Pr (Z > C_1 - \gamma \alpha kpr - (1 - \alpha)kpr - kpq)\\
    &= \Pr \left(Z > - (\gamma -1) \alpha kpr + \sqrt{kp(r+q)(L - 2\ln(1-\nu_1))}\right)\\
    &\leq \exp\left(-(\sqrt{(r+q)(L - 2\ln(1-\nu_1))} - (\gamma -1) \alpha r\sqrt{kp})^2/2 \right)
\end{align*} Rearranging we have \[\sqrt{2\ln(1/\nu_1)} \leq \sqrt{(r+q)(L - 2\ln(1-\nu_1))} - (\gamma -1) \alpha r\sqrt{kp}\] Thus, so as long as \[\gamma \leq 1 + \frac{\sqrt{(r+q)L} - \sqrt{2\ln((1+r)/r\alpha)}}{\alpha r \sqrt{kp}}\] we will have $\nu_1 \leq r\alpha/(1+r)$. 
On any round $t$ after the first, a neuron in $\stimcore{A} \setminus C_{t-1}$ will receive \[(1-\alpha)kpr + \gamma\alpha kpr + \gamma \nu_{t-1} kp + (1 - \nu_{t-1})kp + kpq\] while the threshold for first-timers is at least \[kp(1 + r + q) + \sqrt{kp(1+r+q)(L+2\ln(1/\mu_t))}\] where $\mu_t \leq e^{-t^2\beta/\beta_0}$, as in the proof of Theorem 1.
The neurons in $\stimcore{A}$ which make the cap on round $t+1$ consist of those that made the previous cap and this one, and those that are first-timers. From Theorem 1, we know $\Pr (i \in C_{t+1} | i \in C_t) \geq 1 - e^{-t^2}$, so take $\Pr (i \in C_{t+1} | i \in \stimcore{A} \cap C_t) \sim \nu_t$. We need only to find $\Pr (i \in C_{t+1} | i \in \stimcore{A} \setminus C_t) = \nu_t'$. On the second round, we have $\nu_1 = \frac{r\alpha}{1+r}$. So, letting $Z\sim \mathcal N(0, kp(\gamma^2 - 1)\alpha(1+r) + kpq)$, it follows that: \begin{align*}
    \Pr (i \in C_2 | i \in \stimcore{A} \setminus C_1)&= \nu_2'\\
    &\leq \Pr \left(Z > -(\gamma - 1)\frac{\alpha(2+r)}{1+r})kpr + \sqrt{kp(1+r+q)(L)}  \right)\\
    &\leq \exp\left(-\frac{1}{2}\left( \sqrt{(1+r+q)(L)} - \frac{\alpha r(2+r)}{1+r})\sqrt{kp} \right)^2 \right)
\end{align*} and so if \[\gamma \leq 1 + \frac{\sqrt{(1+r+q)L} - \sqrt{2\ln((1+r)/\alpha)}}{\alpha \frac{r^2+2r}{1+r}\sqrt{kp}}\] will ensure that $\nu_2' \leq \alpha/(1+r)$, which is very nearly the bound for the first cap. Thus, no more than an $\alpha$ fraction of neurons in the second cap are in $\stimcore{A}$.

Now, we seek to show that if $\nu_t \leq \alpha$, then we will have $\nu_{t+1} \leq \alpha + 1/k$, since this will ensure that no new neurons from $\stimcore{A}$ have entered the cap. Reasoning as before, let $\Pr (i \in C_{t+1} | i \in \stimcore{A} \setminus C_t) = \nu_{t+1}'$, and then we have \begin{align*}
    \Pr (i \in C_{t+1} | i \in \stimcore{A} \setminus C_t) &= \nu_{t+1}'\\
    &\leq \Pr \left(Z > -2(\gamma - 1)\alpha kpr + \sqrt{kp(1+r+q)(L+2\ln(1/\mu_t))}\right)\\
    &\leq \exp\left(-\left(\sqrt{(1+r+q)(L+2\ln(1/\mu_t))} - 2(\gamma - 1)\alpha r\sqrt{kp}\right)^2/2 \right)
\end{align*} Then solving for $\gamma$, we find that we will have $\nu_t' \leq 1/k$ as long as \[\gamma \leq 1 + \frac{\sqrt{(1+r+q)(L+2t^2)} - \sqrt{2\ln(k)}}{2\alpha r\sqrt{kp}}\] which is the least upper bound so far. Thus, taking $r+q\sim 1$, so long as \[\gamma \leq 1 + \frac{\sqrt{L} - \sqrt{2\ln((1+r)/r\alpha)}}{\alpha r \sqrt{kp}}\] the overlap of any cap with $A^*$ will never exceed $\alpha k$ neurons. By Theorem 1, an assembly $B$ will form with high probability after $\ln(k)$ examples, and we conclude that $|\assmcore{A} \cap \assmcore{B}| \le \alpha k$.
\hfill$\blacksquare$

\subsection*{Proof of Theorem \ref{theorem:classify}} \label{appendix:classify}

Let $\mu$ be the fraction of first-timers included in the cap $C_1$, and $\nu$ be the fraction of neurons in $\assmcore{A}$ in the cap. A neuron in $\assmcore{B}$ receives \[X_b = \gamma kpr + kpq + Z_b\] where $Z_b \sim \mathcal N(0, \gamma^2 kpr + kpq)$ while a neuron in $\assmcore{A}$ receives \begin{align*}
    X_a &= \gamma \alpha kpr + \gamma (1-\alpha) kp\left(\frac{qk}{n}\right) + kpq + Z_a\\ 
    &= \gamma \alpha kpr + kpq + Z_a + O(1)
\end{align*} where $Z_a \sim \mathcal N(0, \gamma^2 \alpha kpr + kpq)$. The threshold to be in the top $\nu k$ of $|\assmcore{A}| \sim k$ draws from this distribution is \[C_1 = \gamma \alpha kpr + kpq + \sqrt{2kp(\gamma^2\alpha r+ q)\ln(1/\nu)}\] Neurons in $\assmcore{B}$ will make the first cap if they exceed this threshold. So, we have \begin{align*}
    \Pr (i \in C_1 | i \in \assmcore{B}) &= 1 - \nu\\
    &\geq \Pr (X_b \geq C_1)\\
    &= \Pr \left(Z_b \geq -\gamma(1- \alpha) kpr + \sqrt{2kp(\gamma^2\alpha r+ q)\ln(1/\nu)} \right)\\
    &\ge 1 - \exp\left(-\frac{1}{2}\frac{\left(\gamma(1- \alpha) r\sqrt{kp} - \sqrt{2(\gamma^2\alpha r+ q)\ln(1/\nu)}\right)^2}{\gamma^2r + q}\right)
\end{align*}
where the last step follows from Lemma \ref{lemma:tail}. Solving for $\nu$ gives \[\nu \leq \exp\left(-\frac{(1- \alpha)^2 kpr}{2(1+\alpha)}\right) \]
Similarly, neurons in neither of the two assembly cores must also exceed the threshold to make the cap, which means \begin{align*}
    \Pr (i \in C_1 | i \not\in \assmcore{A} \cup \assmcore{B}) &= \mu\\
    &\leq \Pr (X > C_1)\\
    &= \Pr \left(Z > (\gamma \alpha - 1) kpr + \sqrt{2kp(\gamma^2\alpha r+ q)\ln(1/\nu)}\right)\\
    &\leq \exp\left(-\frac{1}{2(r+q)}\left((\gamma \alpha - 1) r\sqrt{kp} + \sqrt{2(\gamma^2\alpha r+ q)\ln(1/\nu)} \right)^2\right)\\
    &\le \exp\left(-\frac{1}{2}(\gamma \alpha - 1)^2 kpr \right) \cdot \nu^{\gamma^2 \alpha}
\end{align*}
Then the fraction of neurons in $C_1 \setminus \assmcore{B}$ will be 
\[\nu + \mu \leq \exp\left(-\frac{(1- \alpha)^2 kpr}{2(1+\alpha)}\right) + \exp\left(-\frac{1}{2}(\gamma \alpha - 1)^2 kpr \right) \] \hfill$\blacksquare$

\subsection*{Proof of Theorem \ref{theorem:halfspace}} \label{appendix:halfspace}
For each round $1, \ldots, t, \ldots$, define $\mu_t$ to be the fraction of neurons firing for the first time on that round, and let $X(t)$ be an example sampled from $\mathcal D_+$. For each neuron $i$, let $W_j^i(t)$ represent the weight of the synapses between neuron $i$ and input neuron $j$ on round $t$. Since each synapse is present independently with probability $p$, the number of synapses (i.e. the norm $\|W^i\|_0$) between the input region and neuron $i$ is a binomial random variable, with expectation $np$. Furthermore, the Chernoff bound shows that the number of synapses is sharply concentrated about its mean as \[\Pr(|\|W^i\|_0 - np | \ge n^{1-\epsilon}p) \le 2\exp\left(-\frac{n^{1-2\epsilon}p}{3}\right)\] for any $\epsilon > 0$. Thus, once normalized, the weight of a particular synapse $W_j^i(0)$ will be between, say, $\frac{1}{np-n^{1/3}p}$ and $\frac{1}{np+n^{1/3}p}$ with high probability. Since both of the quantities approach $\frac{1}{np}$ for large enough $n$, we will regard the weight of every synapse at the outset as $\frac{1}{np}$, so that $\E W_j^i(0) = \frac{p}{np} = \frac{1}{n}$.

If $i$ has not yet fired, then the round $t$ input $W^i(t) \cdot X(t)$ is approximately normal, with mean \[\E(W^i(t) \cdot X(t)) = \sum_{j=1}^n \E(W_j^i(t)) \E(X_j(t)) \ge \sum_{j=1}^n \frac{1}{n} \left(\frac{k}{n} + \Delta v_j\right) = \frac{k}{n} + \frac{\Delta}{n} \|v\|_1\]
and variance approximately $\frac{k}{n^2p}$. So, using Lemma \ref{lemma:tail}, a neuron will be in the top $\mu_t k$ of the approximately $n$ neurons which have never fired before if its input from $X(1)$ exceeds 
\[C_1 = \frac{k}{n} + \frac{\Delta}{n} \|v\|_1 + \frac{\sqrt{kpL}}{np}\] 
where $L = 2\ln(n/k)$, and including input from the previous cap $N(k/n, k/n^2p)$ on subsequent rounds,
\[C_t = \frac{2k}{n} + \frac{\Delta}{n} \|v\|_1 + \frac{1}{np}\sqrt{2kp(L + 2\ln(1/\mu_t))}\]
Now, let $i$ be a neuron which made the first cap. Then $W^i(1) \cdot X(1) \ge C_1$ and each nonzero component in $W^i(1)$ will be increased by a factor of $1+\beta$ if the corresponding component of $X(1)$ was nonzero as well. By Lemma \ref{lemma:dotprod}, the conditional expectation of $W_j^i(2)$ will be 
\begin{align*}
    \E(W_j^i(2) \vert W^i(1) \cdot X(1) \ge C_1) &\ge \left(1 + \beta \Pr(X(1)=1)\right) \E(W_j^i(1))\\
    &\ge \frac{1}{n} + \beta \frac{1}{n}\left(\frac{k}{n} + \Delta v_j\right)
\end{align*}
Then we have \begin{align*}
    \E(W^i(2) \cdot X(2) &\vert i \in C_1) = \sum_{j=1}^n \E(W_j^i(2) \vert i \in C_1) \E(X_j(2))\\
    &\ge \sum_{j=1}^n \frac{1}{n}\left(\frac{k}{n} + \Delta v_j\right) + \beta \frac{1}{n}\left(\frac{k}{n} + \Delta v_j\right)^2\\
    &\ge \frac{k}{n} + \frac{\Delta}{n}\|v\|_1 + \frac{\beta\Delta^2}{n}
\end{align*}
Then letting $Y \sim \mathcal N(k/n, k/n^2p)$ denote the input from the previous cap, we have \begin{align*}
    \Pr(i \in C_2 \vert i \in C_1) &= 1 - \mu_2\\
    &\ge \Pr(W^i(2) \cdot X(2) + Y \ge C_2)
\end{align*} Taking $Z \sim \mathcal N(0, k/n^2p)$ to be an underestimate of the variance, \begin{align*}
    1 - \mu_2 &\ge \Pr\left(Z \ge \frac{1}{np}\sqrt{2kp(L + 2\ln(1/\mu_2))} - \frac{\beta\Delta^2}{n}\right)\\
    &\ge 1 - \exp\left(-\frac{1}{2kp}\left(\beta\Delta^2p - \sqrt{2kp(L + 2\ln(1/\mu_2))}\right)^2\right)
\end{align*} following from the tail bound in Lemma \ref{lemma:tail}. Then so as long as 
\[ \Delta^2\beta \ge \frac{\sqrt{2k(L + 2)} + \sqrt{2k}}{\sqrt{p}} \] we will have $\mu_2 \le 1/e$.

Now, suppose $i$ fired on all of the first $t$ rounds. Input neuron $j$ is expected to fire at least $(\frac{k}{n} + \Delta v_i)t$ times, so by Lemma \ref{lemma:lognormal} the weight of an extant synapse $W_j^i(t+1)$ will be concentrated about its mean, which is at least $\frac{1}{np}(1+\beta)^{(\frac{k}{n} + \Delta v_i)t}$. The expected input $W^i(t+1) \cdot X(t+1)$ is thus \begin{align*}
    \E(W^i(t+1) \cdot X(t+1) &| i \in C_1 \cap \ldots \cap C_t) = \sum_{j=1}^n \E(W_j^i(t+1) \vert i \in C_1 \cap \ldots \cap C_t) \hat\E(X_j(t+1))\\
    &\ge \sum_{j=1}^n \frac{1}{n}\left(\frac{k}{n} + \Delta v_j\right)  + \beta t \frac{1}{n}\left(\frac{k}{n} + \Delta v_j\right)^2\\
    &\ge \frac{k}{n} + \frac{\Delta}{n} \|v\|_1 + \frac{\Delta^2}{n} \beta t
\end{align*} Then including recurrent input $Y \sim \mathcal N(kp, kp)$, we have \begin{align*}
    \Pr(i \in C_{t+1} \vert i \in C_1 \cap \ldots \cap C_t) &= 1 - \mu_{t+1}\\
    &\ge \Pr(W^i(t+1) \cdot X(t+1) + Y \ge C_{t+1})
\end{align*} Again taking $Z \sim \mathcal N(0, k/n^2p)$ to be the variance, we have \begin{align*}
    1 - \mu_{t+1} &\ge \Pr\left(Z \ge \frac{1}{np}\sqrt{2kp(L + 2\ln(1/\mu_{t+1}))} - \frac{\Delta^2}{n} \beta t \right)\\
    &\ge 1 - \exp\left(-\frac{1}{2kp}\left(\Delta^2 \beta tp - \sqrt{2kp(L + 2\ln(1/\mu_{t+1}))}\right)^2\right)
\end{align*} again following from the tail bound. Then we will have $\mu_{t+1} \le e^{-t}$ as long as \[ \Delta^2 \beta \ge \frac{\sqrt{2k(L+2t)} + \sqrt{2kt}}{t\sqrt{p}} \] which is certainly no more than $\beta_0$.

So, if $\beta \ge \beta_0$, the probability of a neuron that made the cap $t$ times missing the next one drops off exponentially. It follows that after $\ln(k)$ rounds the process will have converged, and each neuron in a particular cap has a probability of at least $1 - \exp(-(\beta/\beta_0)^2))$ of making every subsequent cap. So, the total support of all caps together is no more than $ k / (1 - \exp(-(\beta/\beta_0)^2))) = k + o(k)$ in expectation.

Now, suppose the training process continues until the cap converges to some set $\assmcore{A}$ of size $k + o(k)$, which takes roughly $\ln(k)$ rounds with high probability. For neuron $i \in \assmcore{A}$, the weight of an extant synapse $W_j^i(\ln(k))$ will be, in expectation, \[W_j^i(\ln(k)) = \frac{1}{np}(1 + \beta)^{\left(\frac{k}{n} + \Delta v_j\right) \ln(k)}\] and on average over the neurons in $\assmcore{A}$, we will have \[\sum_{j=1}^n W_j^i \ge \sum_{j=1}^n \frac{p}{np}(1 + \beta)^{\left(\frac{k}{n} + \Delta v_j\right) \ln(k)} \ge 1 + \frac{\beta \ln(k)}{n} \Delta \|v\|_1 \] The neurons are then allowed to return to rest, and each neuron renormalizes the weights of its incoming synapses to sum to 1, so that \[W_j^i = \frac{(1 + \beta)^{\left(\frac{k}{n} + \Delta v_j\right) \ln(k)}}{np + \beta p \ln(k) \Delta \|v\|_1} + o(n^{-1})\]
Choose examples $X^+$ and $X^-$ at random, so that \[v^\top X^+ \ge \frac{k}{n} \|v\|_1 + \Delta \quad\quad v^\top X^- \le \frac{k}{n} \|v\|_1 - \Delta\] Letting $C_+, C_-$ denote the caps formed after presenting the respective example once, define $\epsilon_+$ to be the fraction of neurons in $C_+$ formed in response to $X^+$ which are not in $\assmcore{A}$, and $\epsilon_-$ to be the fraction of neurons in $C_-$ which are in $\assmcore{A}$, i.e. \[\epsilon_+ = \frac{|C_+ \setminus \assmcore{A}|}{k} \quad\quad \epsilon_- = \frac{|C_- \cap \assmcore{A}|}{k}\]
As before, the input for a neuron $i$ outside of $\assmcore{A}$ will be nearly normal with expectation
\begin{align*}
    \E(W^i \cdot X^+) &= \sum_{j=1}^n \E(W_j^i) \E(X_j^+)\\
    &\ge \sum_{j=1}^n \frac{1}{n} \left( \frac{k}{n} + \Delta v_j \right)\\
    &= \frac{k}{n} + \frac{\Delta}{n} \|v\|_1
\end{align*} and variance nearly $Z \sim \mathcal N(0, k/n^2p)$, so that the threshold to make the top $\epsilon_+ k$ neurons will be at least \[C_+ = \frac{k}{n} + \frac{\Delta}{n} \|v\|_1 + \frac{1}{np} \sqrt{kp(L + 2\ln(1/\epsilon_+))}\]
Now, consider $i \in \assmcore{A}$. Its input in expectation will be \begin{align*}
    \E(W^i \cdot X^+ \vert i \in \assmcore{A}) &= \sum_{j=1}^n \E(W_j^i) \E(X_j^+)\\
    &\ge \sum_{j=1}^n \frac{\frac{k}{n} + \Delta v_j}{np + \beta p \ln(k) \Delta \|v\|_1}  + \frac{\beta \ln(k)\left(\frac{k}{n} + \Delta v_j\right)^2}{np + \beta p \ln(k) \Delta \|v\|_1}\\
    &\ge \frac{k + \Delta \|v\|_1 + \beta \ln(k)\Delta^2}{np + \beta p \ln(k) \Delta \|v\|_1}
\end{align*} while its variance will be \begin{align*}
    \var(W^i \cdot X^+ \vert i \in \assmcore{A}) &= \E((W^i \cdot X^+)^2 \vert i \in \assmcore{A}) - \E(W^i \cdot X^+ \vert i \in \assmcore{A})^2\\
    &\ge \sum_{j=1}^n p\left(\frac{k}{n} + \Delta v_j\right) \frac{(1+\beta)^{(\frac{k}{n} + \Delta v_j) 2\ln(k)}}{(np + \beta \ln(k) p \Delta \|v\|_1)^2}\\
    &\ge \frac{k}{(n + \beta \ln(k) \Delta \|v\|_1)^2p}
\end{align*} Then letting $Z \sim \mathcal N(0, \frac{k}{(n + \beta \ln(k) \Delta \|v\|_1)^2p})$ denote the uncertainty, we have \begin{align*}
    \Pr(i \in C_+ &\vert i \in \assmcore{A}) = 1 - \epsilon_+\\
    &\ge \Pr(W^i \cdot X^+ \ge C_+)\\
    &\ge \Pr\left(Z \ge \beta \ln(k)  \frac{\frac{\Delta \|v\|_1}{n}\left(k + \Delta \|v\|_1\right) - \Delta^2}{n + \beta \ln(k) \Delta \|v\|_1} + \frac{1}{np} \sqrt{kp(L + 2\ln(1/\epsilon_+))}\right)
\end{align*} Then using Lemma \ref{lemma:tail}, \begin{align*}
    \epsilon_+ \le \exp\left(-\tfrac{1}{2kp}\left(p\beta \ln(k)  \left(\Delta^2 - \tfrac{\Delta \|v\|_1\left(k + \Delta \|v\|_1\right)}{n}\right) -  \tfrac{n + \beta \ln(k) \Delta \|v\|_1^2}{n} \sqrt{kp(L + 2\ln(1/\epsilon_+))}\right)^2\right)
\end{align*} So, for fixed $\epsilon_+$ we need \[\Delta^2 - \frac{\Delta \|v\|_1}{n}\left(k + \Delta \|v\|_1\right) \ge \frac{\left(1 + \beta \ln(k)\frac{\Delta \|v\|_1}{n}\right) \sqrt{kp(L + 2\ln(1/\epsilon_+))} + \sqrt{2kp\ln(1/\epsilon_+)}}{p\beta \ln(k)} \] Now, note that if $\Delta \ge \frac{2k}{\sqrt{n}}$, recalling that $\|v\|_1 \le \sqrt{n}/2$, we have must have $\Delta \ge \frac{k\|v\|_1}{n/2 - \|v\|_1^2}$, and so \begin{align*}
    \Delta^2 - \frac{\Delta \|v\|_1}{n}\left(k + \Delta \|v\|_1\right) &= \Delta^2 - \left(\Delta \frac{k\|v\|_1}{n/2 - \|v\|_1^2} \frac{n/2 - \|v\|_1^2}{n} + \frac{\Delta^2 \|v\|_1^2}{n}\right)\\
    &\ge \Delta^2 - \left(\frac{\Delta^2}{2} - \frac{\Delta^2 \|v\|_1^2}{n} + \frac{\Delta^2 \|v\|_1^2}{n} \right)\\
    &= \frac{\Delta^2}{2}
\end{align*} Additionally, if $\Delta \ge 2 \sqrt{\frac{k}{np}} \sqrt{L + 2\ln(1/\epsilon_+)}$ (a much milder bound compared to the previous one, for $p = \Omega(k^{-1})$ and $\epsilon_+$ fixed) then it suffices that \[\Delta^2 \beta \ge \frac{4\sqrt{k}}{\ln(k) \sqrt{p}} \left(\sqrt{L + 2\ln(1/\epsilon_+)} + \sqrt{2\ln(1/\epsilon_+)}\right) \]

Now, consider the example $X^-$. The input to a neuron $i$ outside of $\assmcore{A}$ will be nearly normal with expectation \begin{align*}
    \hat\E(W^i \cdot X^-) &= \sum_{j=1}^n \E(W_j^i) \E(X_j^-)\\
    &= \sum_{j=1}^n \frac{k}{n^2}\\
    &= \frac{k}{n}
\end{align*}
For $i \in \assmcore{A}$, \begin{align*}
    \E(W^i \cdot X^- \vert i \in \assmcore{A}) &= \sum_{j=1}^n \hat\E(W_j^i) \E(X_j^-)\\
    &\ge \sum_{j=1}^n \frac{(1 + \beta)^{(\frac{k}{n} + \Delta v_j)\ln(k)}}{\sum (1+\beta)^{(\frac{k}{n} + \Delta v_l)\ln(k)}} \frac{k}{n}\\
    &= \frac{k}{n}
\end{align*} with variance no larger than $\frac{k}{n^2p}$. It follows that the threshold for a neuron in $\assmcore{A}$ to make the top $\epsilon_-k$ of the $k$ neurons in $\assmcore{A}$ will be no more than 
\[C_- = \frac{k}{n} + \frac{\sqrt{2kp\ln(1/\epsilon_-)}}{np}\] So, for $Z \sim \mathcal N(0, k/n^2p)$ we have \begin{align*}
    \Pr(i \in C_- &\vert i \not\in \assmcore{A}) = (1 - \epsilon_-)\frac{k}{n}\\
    &\ge \Pr\left(W^i \cdot X^- \ge C_-)\right)\\
    &\ge \Pr\left(Z \ge \frac{\sqrt{2kp\ln(1/\epsilon_-)}}{np}\right)\\
    &= \epsilon_-
\end{align*} 
Rearranging shows that \[\epsilon_- \le \frac{\frac{k}{n}}{1 + \frac{k}{n}} \le \frac{k}{n}\] Now, comparing all of above bounds on $\Delta$ when $\epsilon_+ \le 1/e$: \begin{align}
    \Delta &\ge \frac{2k}{\sqrt{n}}\\
    \Delta &\ge 2\sqrt{\frac{k}{np}} \sqrt{L + 2}\\
    \Delta^2 \beta &\ge \frac{4\sqrt{k}}{\beta \ln(k) \sqrt{p}} \left(\sqrt{L + 2} + \sqrt{2}\right)\\
    \Delta^2 \beta &\ge \frac{\sqrt{2k(L + 2)} + \sqrt{2k}}{\sqrt{p}}
\end{align} Since $k \le n$, and assuming $\beta$ is no larger than a constant, the bound in (4) is the strongest. Thus, taking \[\Delta^2 \beta \ge \frac{\sqrt{2k(L + 2)} + \sqrt{2k}}{\sqrt{p}}\] and $\Delta \ge \sqrt{L + 2}$ is sufficient to ensure that the assembly creation process converges within $\ln(k)$ steps, and $\epsilon_+, \epsilon_- \le 1/e$. It follows that if more than half (resp. less than half) of the neurons in $\assmcore{A}$ fire in response to any given example, it is a positive (resp. negative) one with high probability. \hfill$\blacksquare$

\end{document}